\documentclass[10pt,twocolumn,letterpaper]{article}

\usepackage{iccv}              

\usepackage{amsmath,amsfonts,bm}

\def\eqref#1{equation~\ref{#1}}
\def\Eqref#1{Equation~\ref{#1}}

\def\1{\bm{1}}

\def\rvd{{\mathbf{d}}}

\def\rvx{{\mathbf{x}}}
\def\rvy{{\mathbf{y}}}
\def\rvz{{\mathbf{z}}}

\def\mI{{\bm{I}}}

\DeclareMathAlphabet{\mathsfit}{\encodingdefault}{\sfdefault}{m}{sl}
\SetMathAlphabet{\mathsfit}{bold}{\encodingdefault}{\sfdefault}{bx}{n}

\def\gL{{\mathcal{L}}}

\def\gN{{\mathcal{N}}}
\def\gO{{\mathcal{O}}}

\def\gU{{\mathcal{U}}}

\newcommand{\E}{\mathbb{E}}

\newcommand{\R}{\mathbb{R}}

\newcommand{\Var}{\mathrm{Var}}

\usepackage{graphicx}
\usepackage{bigstrut, tabularx, multirow, makecell, diagbox}
\usepackage{arydshln}
\usepackage{caption}
\usepackage{subcaption}
\usepackage{zref}
\usepackage{xcolor}
\usepackage{booktabs} 
\usepackage{amssymb}
\usepackage{enumerate}
\usepackage{natbib}
\usepackage{pifont}
\usepackage{amssymb}
\newcommand{\cmark}{\ding{51}}%
\newcommand{\xmark}{\ding{55}}%
\usepackage{tabulary}
\usepackage{url}            
\usepackage{amsfonts}       
\usepackage{nicefrac}       
\usepackage{microtype}      
\usepackage{booktabs}
\usepackage{wrapfig}
\usepackage{bbm, comment}
\usepackage{color}
\usepackage{float}
\usepackage{thm-restate}
\usepackage{fancyhdr} 
\usepackage{algorithm}
\usepackage{algpseudocode}
\usepackage[utf8]{inputenc}

\usepackage[toc,page,header]{appendix}
\usepackage{minitoc}

\usepackage{float}

\newenvironment{proofs}{%
  \proof}{\endproof}

\usepackage{amsmath,amsfonts,bm,amssymb,mathtools,amsthm}
\usepackage{color,xcolor}
\usepackage{amsmath}
\usepackage{xspace} 
\def\onedot{$\mathsurround0pt\ldotp$}

\def\ie{\emph{i.e}\onedot, }
\newtheorem{theorem}{Theorem}

\newtheorem{lemma}{Lemma}

\def\eqref#1{equation~\ref{#1}}
\def\Eqref#1{Eq.~(\ref{#1})}

\def\1{\bm{1}}

\def\rvd{{\mathbf{d}}}

\def\rvx{{\mathbf{x}}}
\def\rvy{{\mathbf{y}}}
\def\rvz{{\mathbf{z}}}

\def\mI{{\bm{I}}}

\DeclareMathAlphabet{\mathsfit}{\encodingdefault}{\sfdefault}{m}{sl}
\SetMathAlphabet{\mathsfit}{bold}{\encodingdefault}{\sfdefault}{bx}{n}

\def\gL{{\mathcal{L}}}

\def\gN{{\mathcal{N}}}
\def\gO{{\mathcal{O}}}

\def\gU{{\mathcal{U}}}

\newcommand\numberthis{\addtocounter{equation}{1}\tag{\theequation}}

\usepackage{varwidth}
\newsavebox\tmpbox
\usepackage[export]{adjustbox}
\usepackage{caption}
\definecolor{myred}{RGB}{191, 2, 19}

\def\setstretch#1{\renewcommand{\baselinestretch}{#1}}
\newcommand{\method}{$f$-distill} 
\newcommand{\methodtext}{$f$-distill } 

\definecolor{iccvblue}{rgb}{0.21,0.49,0.74}
\usepackage[pagebackref,breaklinks,colorlinks,allcolors=iccvblue]{hyperref}

\def\paperID{14795} 
\def\confName{ICCV}
\def\confYear{2025}

\title{One-step Diffusion Models with $f$-Divergence Distribution Matching}

\author{Yilun Xu\\
NVIDIA\\
{\tt\small yilunx@nvidia.com}
\and
Weili Nie\\
NVIDIA\\
{\tt\small wnie@nvidia.com}
\and 
Arash Vahdat \\
NVIDIA\\
{\tt\small avahdat@nvidia.com}
}

\begin{document}
\thispagestyle{plain}
\maketitle

\begin{abstract}
Sampling from diffusion models involves a slow iterative process that hinders their practical deployment, especially for interactive applications. To accelerate sampling, recent approaches distill a multi-step diffusion model into a single-step student generator via variational score distillation, which matches the distribution of samples generated by the student to the teacher's distribution. However, these approaches use the reverse Kullback–Leibler (KL) divergence for distribution matching which is known to be mode-seeking. In this paper, we generalize the distribution matching approach using a novel $f$-divergence minimization framework, termed \method, that covers different divergences with different properties.  
We derive the gradient of the $f$-divergence between the teacher and student distributions and show that it is expressed as the product of their score differences and a weighting function determined by their density ratio. This weighting function naturally emphasizes samples with higher density in the teacher distribution, when using a less mode-seeking divergence. We observe that the popular variational score distillation approach using the reverse-KL divergence is a special case within our framework. Empirically, we demonstrate that alternative $f$-divergences, such as forward-KL and Jensen-Shannon divergences, outperform the current best variational score distillation methods across image generation tasks. In particular, when using Jensen-Shannon divergence, \methodtext achieves current state-of-the-art one-step generation performance on ImageNet64 and zero-shot text-to-image generation on MS-COCO.
\looseness=-1
\end{abstract}    
\section{Introduction}
\label{sec:intro}

Diffusion models~\cite{ho2020ddpm, song2020score} are transforming generative modeling in visual domains, with impressive success in generating images~\cite{rombach2022high,saharia2022photorealistic, balaji2022ediffi}, videos~\cite{ho2022video, singer2023makeavideo}, 3D objects~\cite{luo2021diffusion, zeng2022lion}, motion~\cite{zhang2022motiondiffuse, yuan2023physdiff}, etc. However, one of the key limitations of deploying diffusion models in real-world applications is their slow and computationally expensive sampling process that involves calling the denoising neural network iteratively.
\begin{figure}[t]
    \centering
  \begin{subfigure}[b]{0.48\textwidth}
    \includegraphics[width=\textwidth]{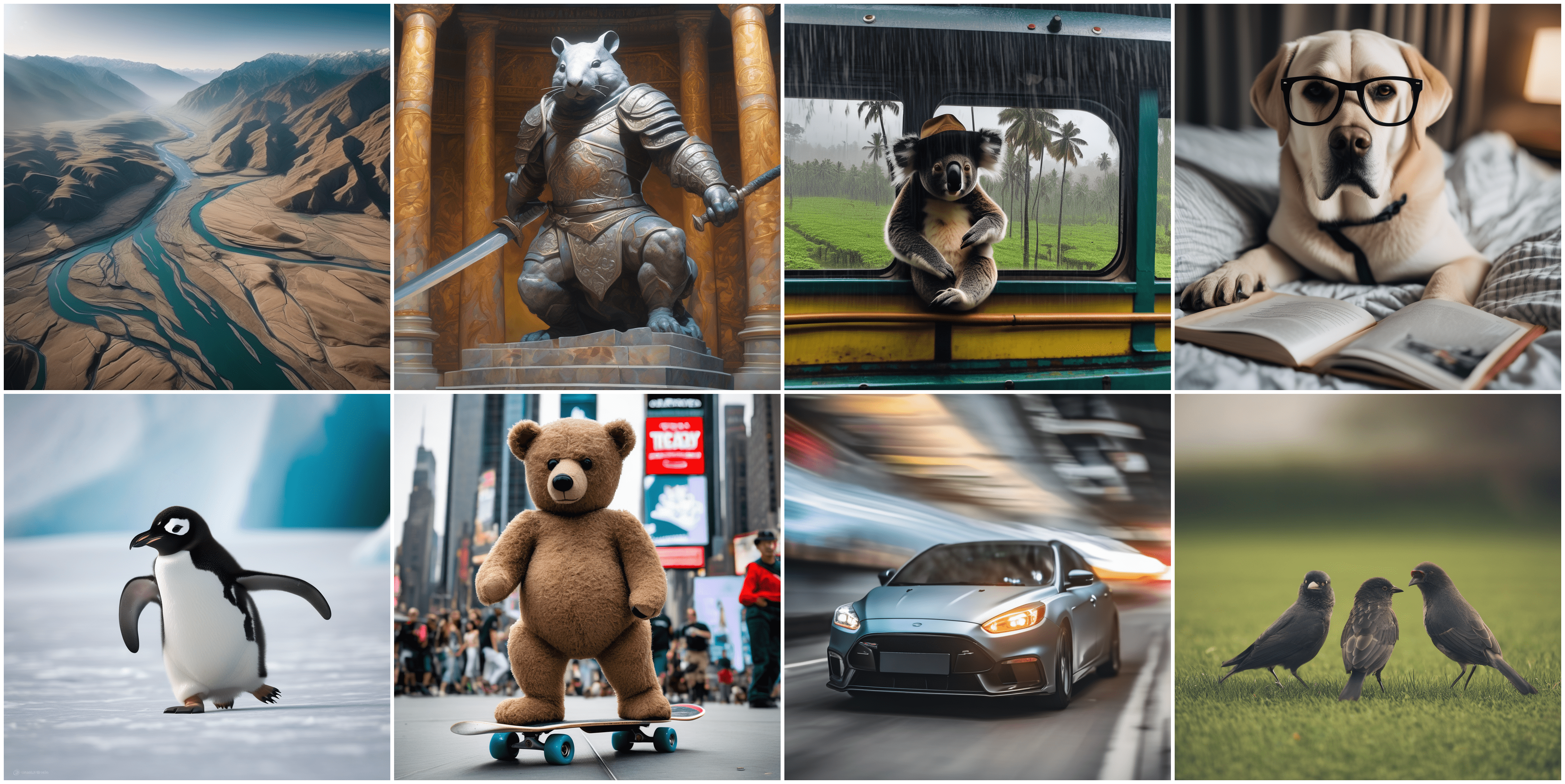}
        \vspace{-12pt}
    \caption{50-step Teacher, SDXL}
  \end{subfigure}
 \hfill
    \begin{subfigure}[b]{0.48\textwidth}
    \includegraphics[width=\textwidth]{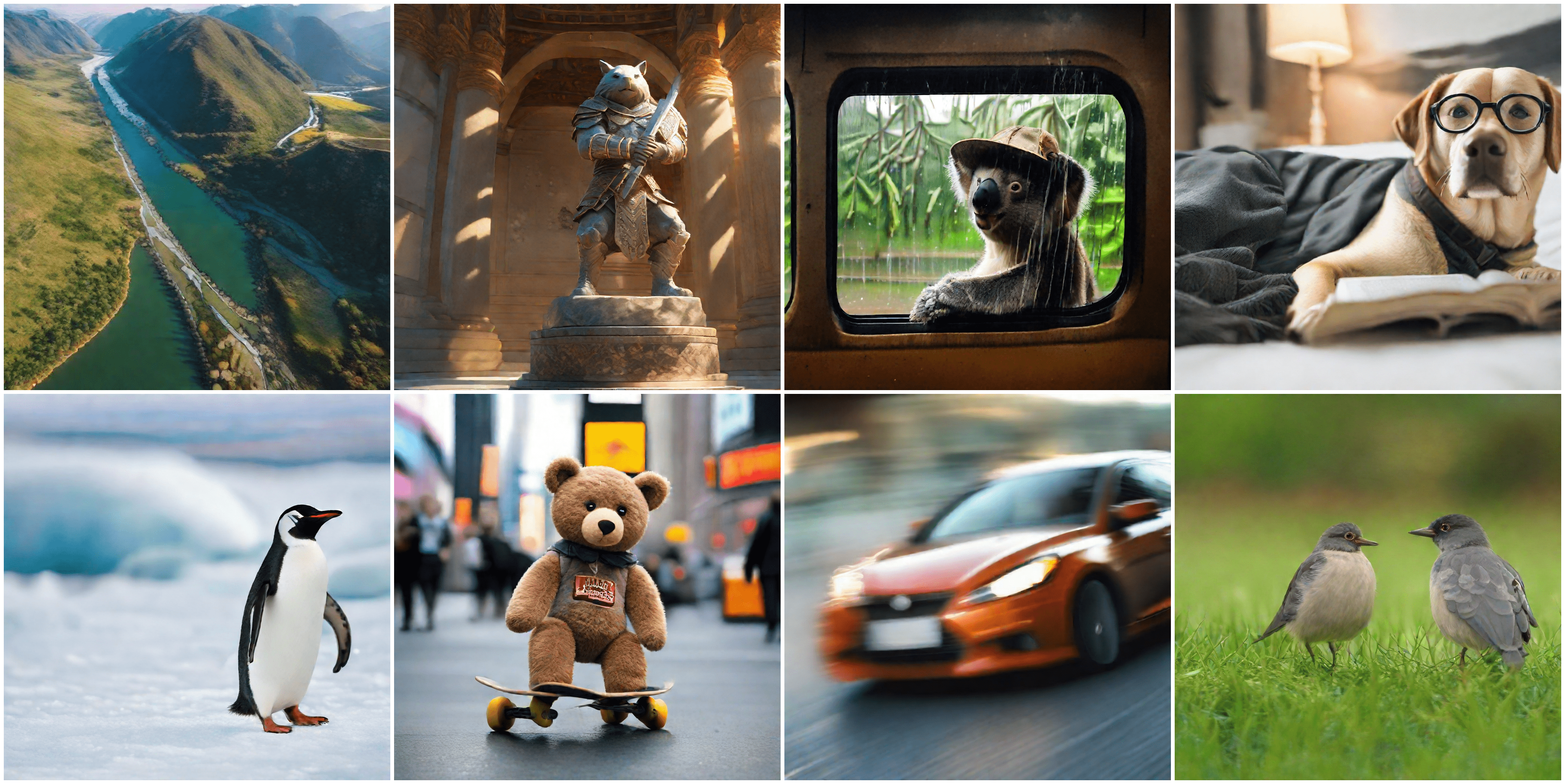}
    \vspace{-12pt}
    \caption{One-step $f$-distill}
  \end{subfigure}
    \caption{Uncurated generated samples by the 50-step teacher~(CFG=8) (a), and one-step student in \methodtext (b), using same set of prompts on SDXL.}
    \label{fig:vis-main}
    \vspace{-12pt}
\end{figure}
Early works on accelerating diffusion models relied on better numerical solvers for solving the ordinary differential equations (ODEs) or stochastic differential equations (SDEs) that describe the sampling process of diffusion models~\cite{Song2020DenoisingDI, JolicoeurMartineau2021GottaGF, lu2022dpm, Karras2022ElucidatingTD, Xu2023RestartSF}. However, these methods can only reduce the number of sampling steps to around tens of steps, due to the discretization error, accumulated with fewer steps.

More recently, distillation-based approaches aim at the ambitious goal of reducing the number of sampling steps to a single network call. These approaches can be generally grouped into two categories: 1) trajectory distillation~\cite{song2023consistency, song2023improved, geng2024consistency, lee2024truncated, lu2024simplifying} which distills the deterministic ODE mapping between noise and data intrinsic in a diffusion model to a one-step student, and 2) distribution matching approaches~\cite{yin2024one, yin2024improved, zheng2024diffusion, zhou2024score} that ignore the deterministic mappings, and instead, matches the distribution of samples generated by a one-step student to the distribution imposed by a pre-trained teacher diffusion model. Among the two categories, the latter often performs better in practice as the deterministic mapping between noise and data is deemed complex and hard to learn. Naturally, the choice of divergence in distribution matching plays a key role as it dictates how the student's distribution is matched against the teacher's. Existing works~\cite{yin2024one, yin2024improved, dao2025swiftbrush, nguyen2024swiftbrush} commonly use variational score distillation~\cite{Wang2023ProlificDreamerHA} that matches the distribution of the student and teacher by minimizing the reverse-KL divergence. However, this divergence is known to be mode-seeking~\cite{Bishop2006PatternRA} and can potentially ignore diverse modes learned by the diffusion model.

The reverse-KL divergence is a member of the broader $f$-divergence family~\cite{Rnyi1961OnMO}. The $f$-divergence represents a large family of divergences including reverse-KL, forward-KL, Jensen-Shannon (JS), squared Hellinger, etc. These divergences come with different trade-offs, including how they penalize the student for missing modes in the teacher distribution and how they can be estimated and optimized using Monte Carlo sampling. However, the application of arbitrary $f$-divergences to diffusion distillation, and the practical estimation of the student's gradient, remain open challenges.  

In this work, we address these challenges by establishing the connection between $f$-divergence and diffusion distillation with a novel distillation framework, which we term \method. We derive the gradient of $f$-divergence distribution matching within the context of diffusion distillation, and show that it is the product of the difference in score between teacher and student (which also exists in prior works), and a weighting function that depends on density ratio and the chosen $f$-divergence (new in this work), as illustrated in Fig.~\ref{fig:teaser}. The density ratio can be readily obtained from the discriminator in the commonly used auxiliary GAN objective. We show that the previous DMD approach is a special case of our approach that corresponds to a constant weighting. We discuss how the newly derived weighting coefficient influences the tradeoffs discussed above and propose normalization techniques for stabilizing divergences with higher gradient variance. As shown in Fig.~\ref{fig:toy_diff}, we observe that the weight coefficient for less mode-seeking $f$-divergences will downweight the score difference in the areas where the teacher has low density. This is in line with the observation that score estimation in low-density regions can be inaccurate~\cite{karras2024guiding} and allows our model to adaptively rely less on matching its score with the teacher's \textit{unreliable} score on such regions. 

We further analyze the properties of various canonical $f$-divergences within our framework. For instance, forward-KL has a better mode coverage, but has a large gradient variance; JS demonstrates moderate mode-seeking and gradient saturation, particularly in early training stages, but exhibits low variance. Our analysis reveals that no single $f$-divergence consistently outperforms others across all datasets. We observe divergences with better mode coverage tendencies generally perform better on the CIFAR-10 dataset. However, on large-scale challenging datasets like ImageNet-64 and text-to-image generation with Stable Diffusion, divergences with lower variance achieve superior results. Empirically, we validate the \methodtext framework on several image generation tasks. Quantitative results demonstrate that the less mode-seeking divergences in \methodtext consistently outperform previous best variational score distillation approaches. Notably, by minimizing the less mode-seeking and lower gradient variance Jensen-Shannon divergences, \methodtext achieves new state-of-the-art one-step generation performance on ImageNet-64 and zero-shot MS-COCO (using SD v1.5). We also show that \methodtext is scalable to larger model SDXL, as shown in Fig~\ref{fig:vis-main}. Our empirical analysis also confirms that the weighting function effectively assigns smaller weights to regions with larger score differences.

\begin{figure*}[t]
    \centering
    \vspace{-0.8cm}
    \includegraphics[width=0.8\linewidth, clip=true, trim={0, 0, 0, 0.5cm,}]{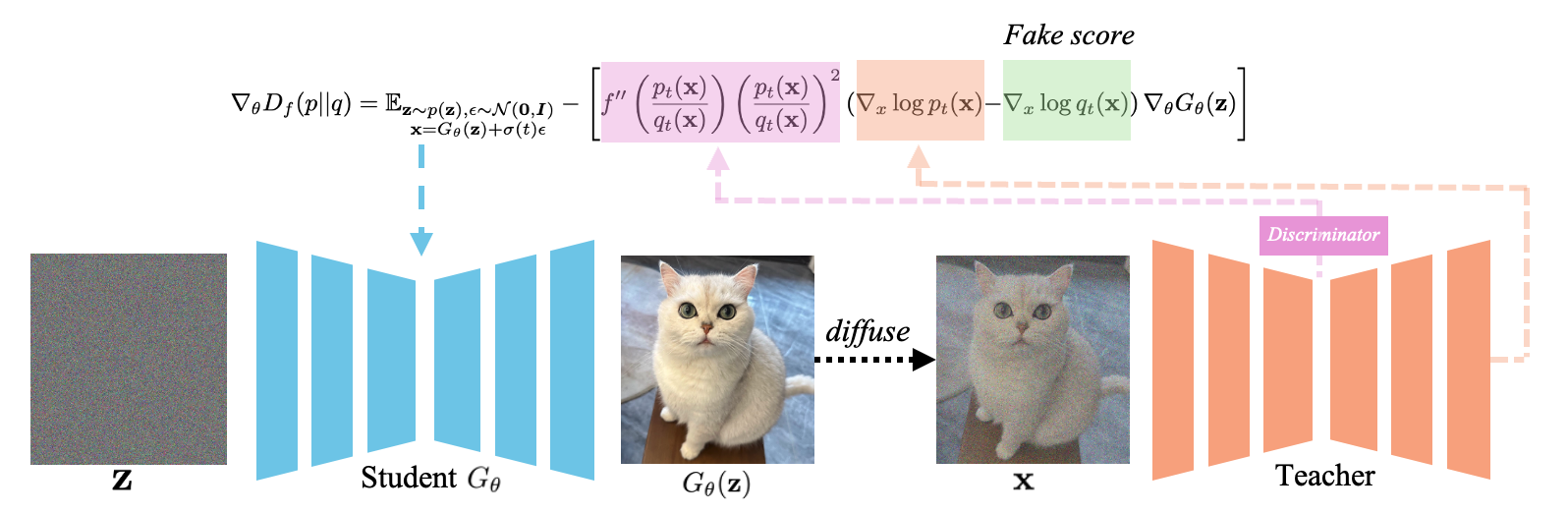}
    \vspace{-8pt}
    \caption{The gradient update in \methodtext is a product of the difference between the teacher and fake scores and a weighting function determined by the chosen $f$-divergence and density ratio. The density ratio is readily available in the auxiliary GAN objective. 
    }
    \label{fig:teaser}
    \vspace{-10pt}
\end{figure*}


 \begin{figure}
     \centering
     \includegraphics[width=0.9\linewidth, trim=0.cm 0.2cm 0.cm 0.2cm, clip]{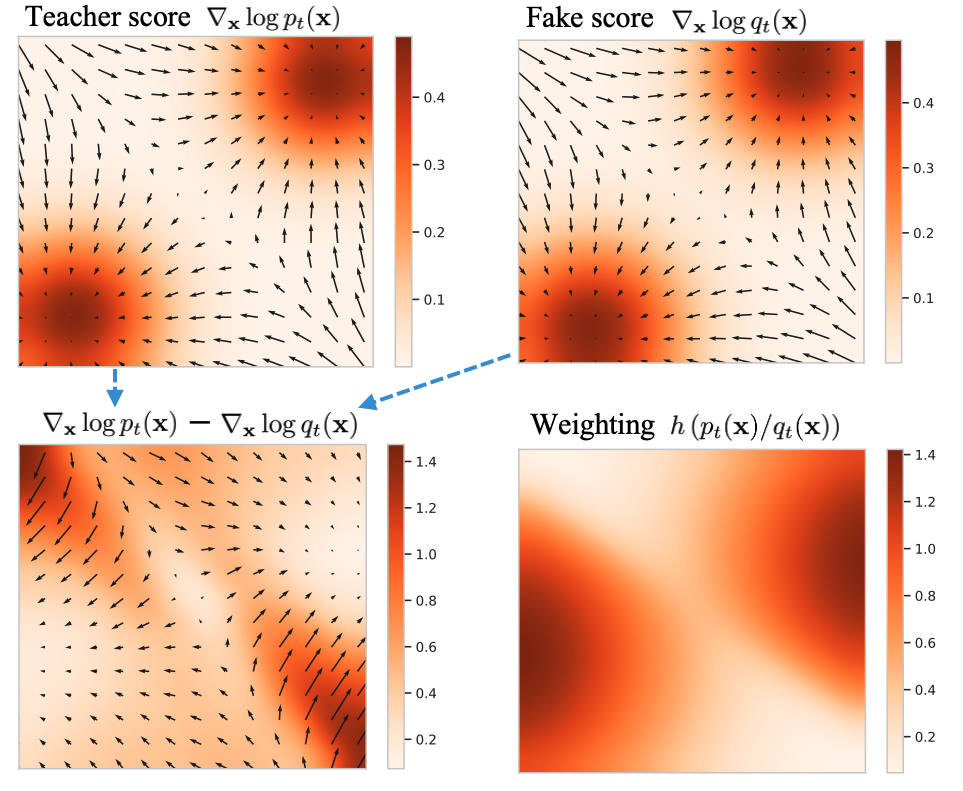}
     \caption{Score difference and the weighting function on a 2D example. $h$ is the weighting function in forward-KL.  Observe that the teacher and fake scores often diverge in lower-density regions (darker colors in the bottom left figure indicate larger score differences), where larger estimation errors occur. The weighting function downweights these regions~(lighter colors in the bottom right figure) during gradient updates for \method.\looseness=-1}
     \label{fig:toy_diff}
     \vspace{-12pt}
 \end{figure}

\textbf{Contributions.} \textit{(i)} We derive the gradient of the $f$-divergence in distribution matching distillation, enabling the application of arbitrary $f$-divergences. \textit{(ii)} We discuss different trade-offs with different choices of $f$-divergence in terms of mode seeking, gradient saturation and variance. \textit{(iii)} We provide practical guidelines on reducing the variance of gradient and estimating different terms in the objective efficiently. \textit{(iv)} We empirically show that our proposed \methodtext achieves the state-of-the-art FID score in one-step generation on the ImageNet-64 and zero-shot MS-COCO text-to-image benchmark. \looseness=-1

\section{Background}

\subsection{Diffusion models}
The goal of \methodtext is to accelerate the generation of pre-trained (continuous-time) DMs~\cite{song2020score, ho2020ddpm}. In this paper, we follow the popular EDM framework~\cite{Karras2022ElucidatingTD} for the notations and forward/backward processes. DMs perturb the clean data {$\rvx_0 \sim p_{\textrm{data}}$} in a fixed forward process using $\sigma^2(t)$-variance Gaussian noise, where $\rvx_0 \in \R^d$ and $t$ denotes the time along the diffusion process. The resulting intermediate distribution is denoted as $p_t(\rvx_t)$ with $\rvx_t \in \R^d$. For notation simplicity, we will use $\rvx$ to replace $\rvx_t$, unless stated otherwise, throughout the paper.
For sufficiently large $\sigma_\textrm{max}$, this distribution is almost identical to pure random Gaussian noise. DMs leverage this observation to sample the initial noise $\epsilon_\text{max} \sim \mathcal{N}( \mathbf{0},\sigma_\textrm{max}^2\mI)$, and then iteratively denoise the sample by solving the following backward ODE/SDE, which guarantees that if $\sigma(0)=0$, the final $\rvx$ follows the data distribution $p_{\textrm{data}}$:
{\small
\begin{equation}\label{eq:generation_sde}
\begin{split}
    & d\rvx = \underbrace{-\dot{\sigma}(t)\sigma(t)\boldsymbol{\nabla}_\rvx \log p_t(\rvx) dt}_{\textrm{Probability Flow ODE}} \\
    &\underbrace{- \beta(t)\sigma^2(t) \boldsymbol{\nabla}_\rvx \log p_t(\rvx) dt + \sqrt{2\beta(t)}\sigma(t)d\boldsymbol{\omega}_t}_{\textrm{Langevin Diffusion SDE}},
\end{split}
\end{equation}
}%
where $\boldsymbol{\omega}_t$ is a standard Wiener process and $\boldsymbol{\nabla}_\rvx \log p_t(\rvx)$ is the \textit{score function} of the intermediate distribution $p_t(\rvx)$. The score function is learned by a neural network $s_\phi(\rvx; \sigma(t))$ trained with the denoising score matching objective~\cite{vincent2011,song2019generative}. In \Cref{eq:generation_sde}, the first term is the Probability Flow ODE,  which guides samples from high to low noise levels. The second term is a Langevin Diffusion SDE, which acts as an equilibrium sampler across different noise levels $\sigma(t)$~\cite{Karras2022ElucidatingTD, Xu2023RestartSF}. This component can be scaled by the time-dependent parameter $\beta(t)$ with $\beta(t)=0$ leading to ODE-based synthesis. 
Although different kinds of accelerated samplers for diffusion ODE~\cite{Song2020DenoisingDI, lu2022dpm, Karras2022ElucidatingTD} and SDE~\cite{JolicoeurMartineau2021GottaGF, Karras2022ElucidatingTD, Xu2023RestartSF} have been proposed, they usually still require $>20$ sampling steps in practice to produce decent samples.

\subsection{Variational score distillation}

A recent line of works~\cite{yin2024one, yin2024improved} aim to distill the teacher diffusion models $s_\phi$ into a single step generator $G_\theta$, through \textit{variational score distillation (VSD)}, which is originally introduced for test-time optimization of 3D objects~\cite{Wang2023ProlificDreamerHA}. The goal is to enable a student model $G_\theta$ to directly map the noise $\rvz$ from the prior distribution $p(\rvz) = \mathcal{N}(\rvz;  \mathbf{0}, \mI)$ to the clean sample $\rvx_0$ at $\sigma=0$ using $\rvx_0 = G_\theta(\rvz)$, effectively bypassing the iterative sampling process. Let $p_\phi$ denote the distribution obtained by plugging in pre-trained diffusion models $s_\phi(\rvx; \sigma(t))$ in \Cref{eq:generation_sde}, and let $q_\theta$ denote the output distribution by the one-step generator $G_\theta$~(in the following text, we drop the subscript in $p_\phi$ and $q_\theta$ for notation simplicity). Then, the gradient update for the generator can be formulated as follows:
\begin{align*}
     \E_{t, \rvz, \epsilon} \left[ \left(s_\phi(\rvx; \sigma(t)) - \nabla_\rvx \log q_\theta(\rvx; \sigma(t))\right) \nabla_\theta G_\theta(\rvz)\right]
     \numberthis \label{eq:vsd-obj}
\end{align*}
where $\rvx = G_\theta (\rvz) + \sigma(t) \epsilon $ and $\epsilon \sim \mathcal{N}(  \mathbf{0}, \mI)$.
Intuitively, the gradient encourages the generator to produce samples that lie within high-density regions of the data distribution. This is achieved through the teacher score term, $s_\phi(\rvx; \sigma(t))$, which guides the generated samples towards areas where the teacher model assigns high probability. To prevent mode collapse, the gradient also incorporates a term that discourages the generator from simply concentrating on a single high-density point in the teacher's distribution. This is done by subtracting the score of the student distribution, $\nabla_\rvx \log q_\theta(\rvx; \sigma(t))$. The gradient update is shown to perform distribution matching by minimizing the reverse-KL divergence between the teacher and student distributions~\cite{poole2023dreamfusion, yin2024one}. 

To estimate the score of the student distribution, previous works~\cite{Wang2023ProlificDreamerHA, yin2024one} have employed another \textit{fake score} network $s_\psi (\rvx, \sigma(t))$ to approximate  $\nabla_\rvx \log q_\theta(\rvx; \sigma(t))$. The fake score network $s_\psi (\rvx, \sigma(t))$ is dynamically updated with the standard denoising score matching loss, where the ``clean'' samples come from the generator $G_\theta$ during training. 
Thus, the VSD training alternates between the generator update and the fake score update, with a two time-scale update rule for stabilized training~\citep{yin2024improved}.
Additionally, to further close the gap between the one-step generator and the multi-step teacher diffusion model, a GAN loss is applied to the VSD training pipeline~\citep{yin2024improved}, where a lightweight GAN classifier takes as input the middle features from the fake score network.



\subsection{$f$-divergence}

In probability theory, an $f$-divergence~\cite{Rnyi1961OnMO} quantifies the difference between two probability density functions, $p$ and $q$.  Specifically, when $p$ is absolutely continuous with respect to $q$, the $f$-divergence is defined as:
\begin{align*}
D_f(p||q) = \int q(\rvx) f\Big(\frac{p(\rvx)}{q(\rvx)}\Big) d\rvx
\end{align*}
where $f$ is a convex function on $(0,+\infty)$ satisfying $f(1)=0$. This divergence satisfies several important properties, including non-negativity and the data processing inequality. Many commonly used divergences can be expressed as special cases of the $f$-divergence by choosing an appropriate function $f$. These include the forward-KL divergence, reverse-KL divergence, Hellinger distance, and Jensen-Shannon (JS) divergence, as shown in Table~\ref{tab:f-div}. In generative learning, $f$-divergence has been widely applied to popular generative models, such as GANs~\citep{nowozin2016f}, VAEs~\citep{wan2020fvi}, energy-based models~\citep{yu2020training} and diffusion models~\citep{tang2024fine}.

\section{Method: general $f$-divergence minimization}

\begin{table*}[t]
\vspace{-0.7cm}
\small
    \centering
    \begin{tabular}{l c c c c c}
 \toprule
  & $f(r)$ & $h(r)$ & Mode-seeking? & Saturation? & Variance \\
  \midrule
  reverse-KL & $-\log r$ & $1$ & Yes & No & - \\
    softened RKL & $(r+1) \log{(\frac12+\frac{1}{2r})}$ & $\frac{1}{r+1}$ & Yes  & No & Low \\
Jensen-Shannon & $r\log r - (r+1) \log \frac{r+1}{2}$ & $\frac{r}{r+1}$ & Medium & Yes & Low \\
  squared Hellinger &
     $1-\sqrt{r}$ & $\frac14 r^{\frac12}$
 & Medium & Yes & Low \\ 
  forward-KL & $r\log r$ & $r$ & No & No & High \\
 Jeffreys & ${(r-1)\log(r)}$ & $r+1$ & No & No & High \\
    \bottomrule
    \end{tabular}
        \caption{Comparison of different $f$-divergences as a function of the likelihood ratio $r :=p(\rvx)/q(\rvx)$}
    \label{tab:f-div}
    \vspace{-12pt}
\end{table*}

In this section, we introduce a general distillation framework, termed \method, based on minimizing the $f$-divergence between the teacher and student distributions. Since the student distribution $q$ is the push-forward measure induced by the one-step generator $G_\theta$, it implicitly depends on the generator's parameters $\theta$. Due to this implicit dependency, directly calculating the gradient of $f$-divergence, $D_f(p||q)$, w.r.t $\theta$ presents a challenge. However, the following theorem establishes the analytical expression for this gradient, revealing that it can be formulated as a weighted version of the gradient employed in variational score distillation. Notably, these weights are determined by the density ratio of the generated samples. We state the theorem more generally by providing the gradient for $p_t$ and $q_t$, where $p_t$ is the perturbed distribution through the diffusion forward process for the teacher's distribution $p$, \textit{i.e.,} $p_t = p_0 * \gN( \mathbf{0},\sigma^2(t)\mI)$~(same for the student distribution $q$). \looseness=-1
\vspace{-2pt}
\begin{restatable}{theorem}{mainthm}
\label{thm-main}
Let $p$ be the teacher's generative distribution, and let $q$ be a distribution induced by transforming a prior distribution $p(\rvz)$ through the differentiable mapping $G_\theta$. Assuming $f$ is twice continuously differentiable, then the gradient of $f$-divergence between the two intermediate distribution $p_t$ and $q_t$ w.r.t $\theta$ is:
{
\small
\begin{align*} 
    &\nabla_\theta D_f(p_t||q_t) = \E_{\substack{\rvz,  \epsilon}}-\Big[f''\left(\frac{p_t(\rvx)}{q_t(\rvx)}\right)\left(\frac{p_t(\rvx)}{q_t(\rvx)}\right)^2\\
    & \quad  \quad \quad \quad \quad \ \ \Big(\underbrace{\nabla_\rvx \log p_t(\rvx)}_{\textrm{teacher score}} - \underbrace{\nabla_\rvx \log q_t(\rvx)}_{\textrm{fake score}} \Big)  \nabla_\theta G_\theta(\rvz)\Big]
    \numberthis \label{eq:time-0-loss}
\end{align*}
}%
where $\rvz \sim p(\rvz), \epsilon \sim \mathcal{N}(  \mathbf{0}, \mI)$ and $ \rvx = G_\theta(\rvz)+\sigma(t)\epsilon $

\end{restatable}
\begin{proofs} For simplicity, we prove the $t=0$ case in the main text. Similar proof applies for any $t>0$.
{\small
\begin{align*}
&\nabla_\theta D_f(p(\rvx)||q(\rvx)) =   \nabla_\theta \int q(\rvx) f(\frac{p(\rvx)}{q(\rvx)}) d\rvx \\
    &=  \underbrace{\int \nabla_\theta q(\rvx) f(\frac{p(\rvx)}{q(\rvx)}) d\rvx}_{I} - \underbrace{\int \nabla_\theta q(\rvx) f'(\frac{p(\rvx)}{q(\rvx)}) \frac{p(\rvx)}{q(\rvx)}  d\rvx}_{II}
\end{align*}
}%
It can be shown that (I) / (II) arises from the term associated with the partial derivative of $f$ with respect to $\rvx$ / $q$, respectively. Above we see that both partial derivatives (I) and (II) are in the form $\int \nabla_\theta q(\rvx) g(\rvx) d\rvx$ where $g$ is a differentiable function that is constant with respect to $\theta$. Assuming that sampling from $\rvx \sim q(\rvx)$ can be parameterized to $\rvx = G_\theta(\rvz)$ for $\rvz \sim p(\rvz)$, we can use the identity $\int \nabla_\theta q(\rvx) g(\rvx) d\rvx = \int p(\rvz) \nabla_\rvx g(\rvx) \nabla_\theta G_\theta(\rvz) d\rvz$. The proof for the identity is provided in the Appendix. Using the identity we can simplify (I) and (II) to:
{
\begin{align*}
I &= \int p(\rvz) f'(\frac{p(\rvx)}{q(\rvx)}) \nabla_\rvx \frac{p(\rvx)}{q(\rvx)} \nabla_\theta G_\theta(\rvz) d\rvz \\
II &= \int p(\rvz) f''(\frac{p(\rvx)}{q(\rvx)}) \frac{p(\rvx)}{q(\rvx)} \nabla_\rvx \frac{p(\rvx)}{q(\rvx)} \nabla_\theta G_\theta(\rvz) d\rvz \\
& + \int p(\rvz) f'(\frac{p(\rvx)}{q(\rvx)}) \nabla_\rvx \frac{p(\rvx)}{q(\rvx)} \nabla_\theta G_\theta(\rvz) d\rvz 
\end{align*}
}%
Putting (I) and (II) in \Eqref{eq:time-0-loss}, we have:
{
\begin{align*}
    \nabla_\theta D_f = &-\int p(\rvz) f''(\frac{p(\rvx)}{q(\rvx)}) \frac{p(\rvx)}{q(\rvx)} \nabla_\rvx \frac{p(\rvx)}{q(\rvx)} \nabla_\theta G_\theta(\rvz) d\rvz \\
= &-\int p(\rvz) f''\left(\frac{p(\rvx)}{q(\rvx)}\right)\left(\frac{p(\rvx)}{q(\rvx)}\right)^2\\
    & \quad \quad \left[\nabla_\rvx \log p(\rvx) - \nabla_\rvx \log q(\rvx)\right] \nabla_\theta G_\theta(\rvz) d\rvz 
\end{align*}
}%
where the last identity is from the log derivative trick. 
\end{proofs}

We defer the completed proofs to App~\ref{app:proof} in the supplementary material. Although the student's generative distribution $q$ depends on the parameter $\theta$, Theorem~\ref{thm-main} provides an analytical expression for the gradient of $f$-divergences between the teachers' and students' generative distributions. This gradient is expressed as the score difference between the teacher's and student's distributions, weighted by a time-dependent factor $f''\left({p_t(\rvx_t)}/{q_t(\rvx_t)}\right)\left({p_t(\rvx_t)}/{q_t(\rvx_t)}\right)^2$ determined by both the chosen $f$-divergence and the density ratio. Crucially, every term in the theorem is tractable
, enabling the optimization of distributional matching through general $f$-divergence minimization. For notation convenience, let $h(r) := f''(r)r^2$ denote the weighting function, and $r_t(\rvx) := p_t(\rvx)/q_t(\rvx)$ denote the density-ratio at time $t$. It is worth noting that the gradient of the variational score distillation~(\Eqref{eq:vsd-obj}) can be recovered as a special case of our framework by setting $h(r)\equiv 1$ 
in \Eqref{eq:time-0-loss}, which corresponds to minimizing the reverse-KL divergence~($f(r)=-\log r $). 

\cite{song2021maximum} also shows a connection between f-divergence and score difference, expressing the former as a time integral of the squared score difference in their Theorem 2. However, our formulation differs in two key aspects: (1) Their objective's gradient necessitates computing a Jacobian-vector product, which can be computationally expensive. 
(2) While their $f$-divergence is expressed as an integral of the score difference over time, our $f$-divergence, $D_f(p_t||q_t)$, depends only on the weighting and score at time $t$. In the following proposition, we further show that if the weighting function $h$ is continuous and non-negative on $(0, +\infty)$, then its product with score difference is the gradient of certain $f$-divergence:
\vspace{-5pt}
\begin{restatable}{proposition}{mainprop}\label{prop:general}
For any function $h$ that is continuous and non-negative on $(0, +\infty)$, the expectation 
$\E_{\substack{\rvz,  \epsilon}}-\left[h\left(r_t(\rvx)\right)\right.\left({\nabla_\rvx \log p_t(\rvx)} - {\nabla_\rvx \log q_t(\rvx)} \right)  \nabla_\theta G_\theta(\rvz)]
$
corresponds to the gradient of an $f$-divergence.
\end{restatable}
Although we limit our study in this paper to canonical forms of $f$-divergence, Proposition~\ref{prop:general} allows us to use any continuous and non-negative scalar function as $h$.

In practice, \cite{yin2024one} suggests performing distributional matching all the time along the diffusion process, as teacher and student will have high discrepancy at smaller times, leading to optimization difficulties. We follow this setup and minimize the $f$-divergence along the whole time range, \ie $\gL(\theta) = \int_0^T w_tD_f(p_t||q_t)dt$, where $w_t$ is a time-dependent weight for equalizing the gradient magnitudes across times. The final objective function for \methodtext is as follows:
{
\begin{align*}
    &\gL_{\textrm{\method}}(\theta) = \E_{t, \rvx}\Big[ \numberthis \label{eq:obj-final} \\
    & \texttt{\color{red}sg\big(} w_t h(r_t(\rvx))({\nabla_\rvx \log p_t(\rvx)} - {\nabla_\rvx \log q_{t}(\rvx)}){\color{red}\big)}^T \rvx\Big]
\end{align*}
}%
where $\rvx {=} G_\theta (\rvz) {+} \sigma(t) \epsilon , \rvz {\sim} p(\rvz), \epsilon {\sim} \mathcal{N}(  \mathbf{0}, \mI)$, and $\texttt{sg}$ stands for stop gradient. The gradient of \Eqref{eq:obj-final} equals the time integral of the gradient of $f$-divergence in Theorem 1~(\Eqref{eq:time-0-loss}). In practice, the score of student distribution $\nabla_\rvx \log q_{t}(\rvx_t)$ is approximated by an online diffusion model $s_\psi (\rvx, \sigma(t))$.  \cite{yin2024improved} augments the variational score distillation loss with a GAN objective to further enhance performance. This is motivated by the fact that variational score distillation relies solely on the teacher's score function and is therefore limited by the teacher's capabilities. Incorporating a GAN objective allows the student generator $G_\theta$ to surpass the teacher's limitations by leveraging real data to train a discriminator $D_{\lambda}$: 
\begin{align*}
    &\gL_{\textrm{GAN}}(\lambda) = \E_{t, \rvx \sim p_{\textrm{data}}, \epsilon_1}[\log D_\lambda(\rvx+ \sigma(t) \epsilon_1)] + \\
    &\qquad\qquad \quad\quad \E_{t, \rvz , \epsilon_2}[\log (1-D_\lambda(G_\theta (\rvz) + \sigma(t) \epsilon_2))]
\end{align*}
where $\rvz {\sim} p(\rvz), \epsilon_1, \epsilon_2 \sim \mathcal{N}(\mathbf{0}, \mI) $. We incorporate the auxiliary GAN objective as in prior work, which offers the additional advantage of providing a readily available estimate of the density ratio $r(\rvx_t)$ required by the weighting function in \Eqref{eq:obj-final}. The density ratio is approximated as follows: $r(\rvx_t)= p_t(\rvx)/q_t(\rvx) \approx p_{\textrm{data}, t}/q_t(\rvx) = D_{\lambda}(\rvx_t, t)/(1-D_{\lambda}(\rvx_t, t))$.  In essence, the GAN discriminator $D_{\lambda}$ provides a direct estimate of the density ratio, facilitating the computation of the weighting function. 





\section{Comparing properties of $f$-divergence}

\label{sec:properties}

In this section, we compare the properties across different distance measures in the $f$-divergence family, in the context of diffusion distillation. We will inspect their three properties: mode-seeking, saturation, and variance during training. We summarize the comparison of different $f$s, and their corresponding weighting function $h$, in Table~\ref{tab:f-div}.

\vspace{-10pt}
\paragraph{Mode-seeking.} Mode-seeking divergences~\cite{Bishop2006PatternRA, pmlr-v206-ting-li23a}, such as reverse-KL, encourage the generative distribution $q$ only to capture a subset of the modes of data distribution. This behavior, however, is undesirable for generative models as it can lead to dropped modes and a loss of diversity in generated samples. This phenomenon is observed in the variational score distillation loss~\cite{lu2024simplifying} used in DMD~\cite{yin2024one, yin2024improved}, which corresponds to minimizing the reverse-KL divergence in \methodtext. One way to characterize mode-seeking behavior is by examining the limit $\lim_{r\to \infty}f(r)/r$~\cite{pmlr-v206-ting-li23a}. A lower growth rate of the limit indicates more mode-seeking~(we defer detailed discussions to Sec~\ref{app:property}). Both reverse-KL and JS divergences exhibit this finite limit, with JS having a higher growth rate (and thus, less mode-seeking behavior). In contrast, forward-KL has an infinite limit, echoing its well-known mode-covering property. 

It is noteworthy that in Table~\ref{tab:f-div}, divergences with a stronger tendency towards mode-seeking also exhibit a slower rate of increase in their weighting function $h(r)$ as $r{\to} \infty$. This behavior stems from the fact that $f''(r)=h(r)/r^2$ also increases more slowly, thus tolerating larger density ratios $p/q$ (\ie allowing $q$ to disregard some modes in $p$), ultimately leading to mode-seeking behavior~\cite{Shannon2020NonsaturatingGT}.
For example, $h$ in JS and forward-KL is an increasing function, while in reverse-KL $h$ stays constant. As a result, the weighting function in less mode-seeking divergence will tend to downweight samples in low-density regions of teacher distribution.  \looseness=-1

\vspace{-10pt}

\paragraph{Saturation.} A challenge encountered by prior generative models, such as GANs~\cite{goodfellow2014generative}, when utilizing $f$-divergence is the issue of saturation.  In the early stages of training, the generative and data distribution are often poorly aligned, resulting in samples from $q$ having very low probability under $p$, and vice versa. Consequently, the density ratio $p/q$ tends to be either extremely large or near zero. This poses optimization issues when divergences have small gradients at both extremes. From Fig.~\ref{fig:f}, we can see that squared Hellinger and JS divergences have smaller gradients at the extremes. Nevertheless, in diffusion distillation literature~\cite{yin2024one, yin2024improved}, the saturation issue is mitigated by initializing the weights of the student model with the pre-trained diffusion models.

\begin{figure}[t]
    \centering
    \vspace{-0.6cm}
  \begin{subfigure}[b]{0.234\textwidth}
    \includegraphics[width=\textwidth, trim=0.cm 0.cm 0.6cm 1.cm, clip]{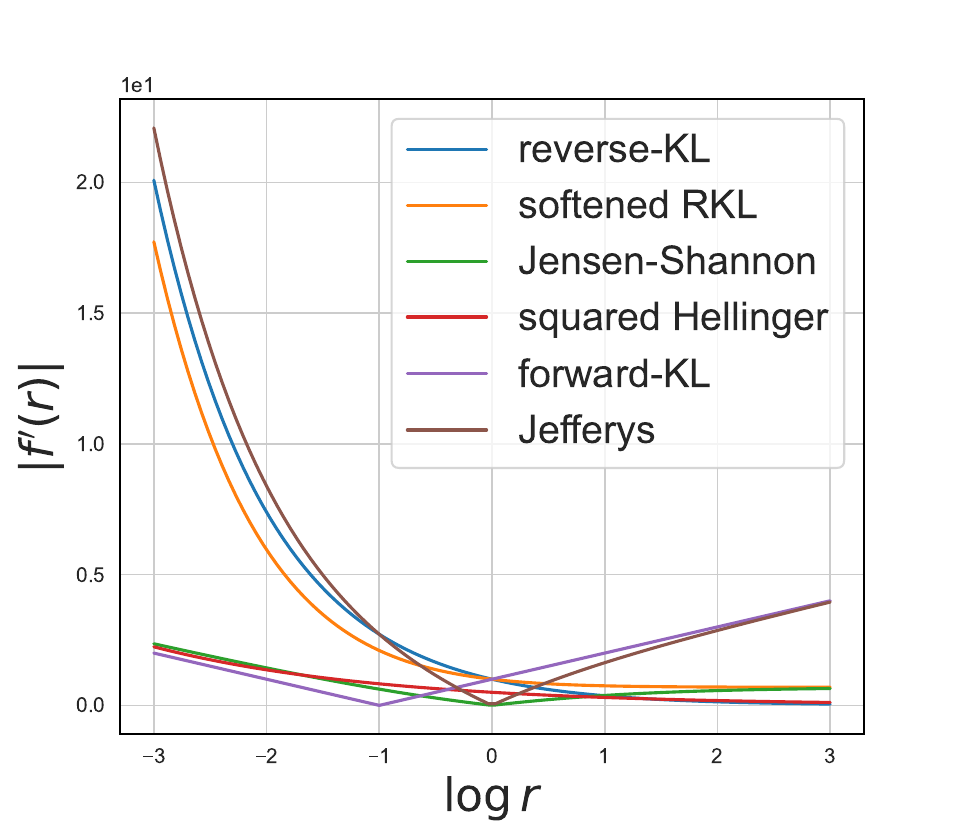}
        \vspace{-9pt}
    \caption{}
        \label{fig:f}
  \end{subfigure}
 \hfill
  \begin{subfigure}[b]{0.234\textwidth}
    \includegraphics[width=\textwidth, trim=0.cm 0.cm 0.6cm 1.cm, clip]{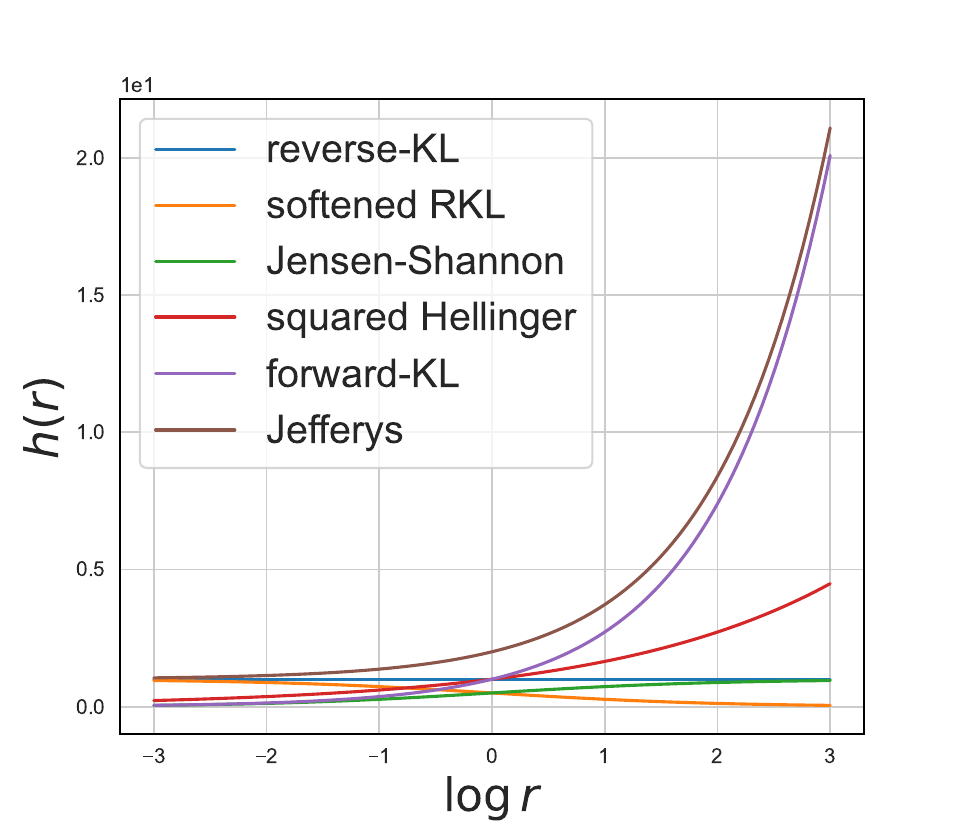}
    \vspace{-9pt}
    \caption{}
        \label{fig:h}
  \end{subfigure}
  \vspace{-6pt}
    \caption{The absolute value of $f'$~\textbf{(a)} and weighting function $h(r)$~\textbf{(b)} in different $f$-divergences.}
    \vspace{-15pt}
\end{figure}

\begin{table}[b]
\vspace{-8pt}
\footnotesize
    \centering
    \begin{tabular}{l c c c}
    \toprule
    & FID $\downarrow$ & Recall $\uparrow$ \\
    \midrule    EDM~\cite{Karras2022ElucidatingTD}~(NFE=35) & 1.79 & 0.63 \\
    Adversarial distillation & 2.60 \\
    \midrule
    \textbf{\method}\\
    reverse-KL~(\textcolor{red}{\cmark}, DMD2 \cite{yin2024improved}) & 2.13 & 0.60 \\
    softened RKL~(\textcolor{red}{\cmark}) & 2.21 & 0.60 \\
    squared Hellinger~(\textcolor{orange}{--}) & 1.99 & 0.63   \\
    JS~(\textcolor{orange}{--}) & 2.00 & 0.62\\
    Jeffreys~(\textcolor{green}{\xmark}) & 2.05 &0.62 \\
    forward-KL~(\textcolor{green}{\xmark}) & 1.92 & 0.62 \\
    \bottomrule
    \end{tabular}
    \vspace{-5pt}
    \caption{FID and Recall scores on CIFAR-10. \textcolor{red}{\cmark}/\textcolor{orange}{--}/\textcolor{green}{\xmark} stand for high/medium/low mode-seeking tendency for $f$-divergence. }
    \label{tab:cifar10}
    \vspace{-6pt}
\end{table}

\begin{figure}
    \centering
  \begin{subfigure}[b]{0.234\textwidth}
    \includegraphics[width=\textwidth, trim=0.2cm 0.cm 0.5cm 0.5cm, clip]{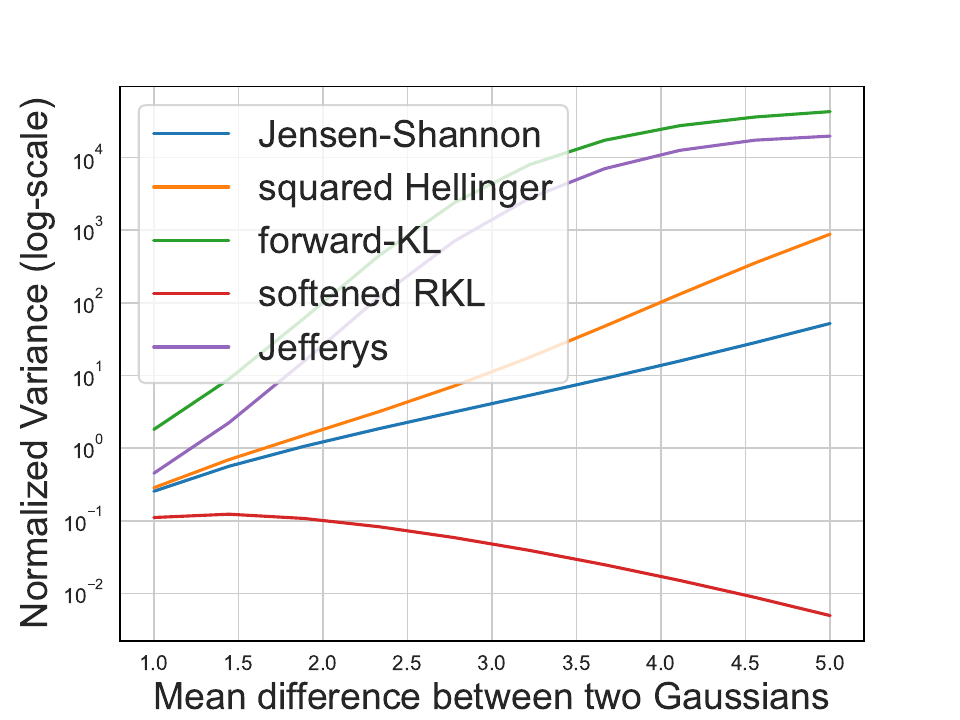}
    \caption{}
     \label{fig:var}
  \end{subfigure}
    \begin{subfigure}[b]{0.234\textwidth}
    \includegraphics[width=\textwidth, trim=0.2cm 0.cm 0.5cm 0.5cm, clip]{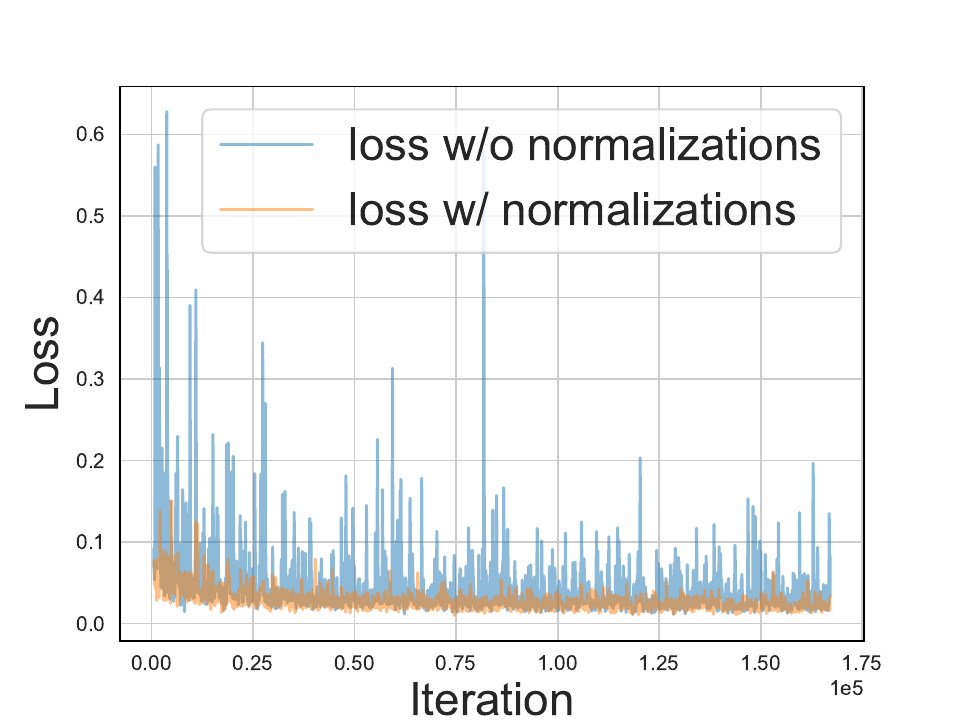}
    \caption{}
      \label{fig:var-loss}
  \end{subfigure}
    \caption{\textbf{(a)} Normalized variance versus the mean difference between two Gaussians. \textbf{(b)} Training losses of forward-KL w/ and w/o normalizations. }
\end{figure}

\begin{table}[t]
\footnotesize
\vspace{-0.5cm}
    \centering
    \begin{tabular}{l c c c}
    \toprule
    & FID $\downarrow$  & Recall $\uparrow$ & NFE\\
    \midrule
    \textbf{Multi-step diffusion models}\\
    EDM~(Teacher)~\cite{Karras2022ElucidatingTD} & 2.35 & 0.68 & 79 \\
    RIN~\cite{Jabri2022ScalableAC} & 1.23 & & 1000\\
    DisCo-Diff~\cite{xu2024disco} & 1.22 & & 623\\
    \midrule
    \textbf{GANs}\\
    BigGAN-deep~\cite{Brock2018LargeSG} & 4.06  & 0.48 & 1 \\
    StyleGAN-XL~\cite{Sauer2022StyleGANXLSS} & 1.52 &  & 1 \\
    \midrule
    \textbf{Diffusion distillation}\\
    DSNO~\cite{zheng2022fast} & 7.83  & 0.61 & 1\\
    iCT-deep~\cite{song2023improved} & 3.25 &0.63 & 1 \\
    Moment Matching~\cite{salimans2024multistep} & 3.00 & & 1 \\
    DMD~\cite{yin2024one} & 2.62 & & 1 \\
    ECM~\cite{geng2024consistency} & 2.49 & &1 \\
    TCM~\cite{lee2024truncated} & 2.20 & &1 \\
    EMD~\cite{xie2024distillation} & 2.20 & 0.59 & 1 \\
    CTM~\cite{kim2024consistency} & 1.92  & 0.57 & 1\\
    Adversarial distillation & 1.88 & & 1\\
    SiD~\cite{zhou2024score} & 1.52  & 0.63 & 1 \\
    GDD-I~\cite{zheng2024diffusion} & 1.16  & 0.60 & 1\\
    & \\[-1.9ex]
    \cdashline{1-4}
    & \\[-1.9ex]
    \textbf{\method}\\
    reverse-KL~(DMD2~\cite{yin2024improved}) & 1.27  & 0.65 & 1\\
    forward-KL~(\textit{ours}) & 1.21 & 0.65  & 1\\
    JS~(\textit{ours}) & \bf{1.16}  & \bf{0.66} & 1\\
    \bottomrule
    \end{tabular}
    \caption{FID score, Recall and NFE on ImageNet-64.}
    \label{tab:imagenet-64}
    \vspace{-15pt}
\end{table}
\vspace{-10pt}
\paragraph{Variance.}
The variance of the weighting function $h$ in the final objective~(\Eqref{eq:obj-final}) is essential to training stability. We use the normalized variance $\Var_q\left(h(p/q)/\E_q[h(p/q)]\right)$ to characterize the variance of different $f$s, ensuring scale-invariant comparison. Fig.~\ref{fig:var} illustrates the normalized variance as a function of the mean difference between two 1D unit-variance Gaussians. The variance of the forward-KL divergence and the Jefferys 
increases significantly as the distance between the Gaussians grows. In contrast, the Jensen-Shannon divergence and the squared Hellinger distance remain relatively stable. This stability contributes to the superior empirical performance of the low-variance Jensen-Shannon divergence in the experimental section.\looseness=-1

To address the high variance often observed in weighting functions for less mode-seeking divergences (see Fig.~\ref{fig:h}), we propose a two-stage normalization scheme. The first stage normalizes the time-dependent density ratio $r_t$, leveraging the fact the expectation of $r_t$ is $1$, \ie $\E_{q_t}[r_t]=1$. We discretize the time range into bins and normalize the values of $r_t$ within each bin by their average. The second stage directly normalizes the weighting function $h$ by its average value within each mini-batch. 
This maintains the relative importance of the \methodtext objective w.r.t the GAN objective, as the scale of weighting function can vary significantly for different $f$s.. Fig.~\ref{fig:var-loss} demonstrates that the loss exhibits a much smaller variance after the above normalization techniques on ImageNet-64 using the forward-KL divergence.  We provide an algorithm box in Alg~\ref{alg:f-distill}. \looseness=-1

\section{Experiment}
\label{sec:experiment}

\subsection{Image generation}

We evaluate \methodtext on CIFAR-10~\citep{krizhevsky2009learning} and ImageNet-64~\citep{deng2009imagenet} for class-conditioned image generation, and on zero-shot MS COCO 2014~\cite{Lin2014MicrosoftCC} for text-to-image generation. We use COYO-700M~\cite{kakaobrain2022coyo-700m} as the training set for text-to-image generation. We use pre-trained models in EDM~\cite{Karras2022ElucidatingTD} as teachers for CIFAR-10 / ImageNet-64, and Stable Diffusion~(SD) v1.5~\cite{rombach2021highresolution} / SDXL~\cite{podell2024sdxl} for text-to-image. As DMD2~\cite{yin2024improved} is a special case~(reverse-KL) under the \methodtext framework, we use the learning rates, CFG guidance weight, update frequency for fake score and discriminator, and the coefficient for GAN loss in generator update in DMD2. 
In the text-to-image experiment, we observe that the estimation of density ratio~(and thus the weighting function $h$) by the discriminator is inaccurate at the early stage of training. To address this, we ``warm up" the discriminator by initializing the model with a pre-trained reverse-KL model, which has a constant $h$. We defer the training details to App~\ref{app:details}. 

Our baseline comparisons include multi-step diffusion models and existing diffusion distillation techniques. We also re-implemented DMD2~\cite{yin2024improved} within our codebase.  Furthermore, to isolate the impact of our proposed \methodtext objective, we conducted an ablation study by removing it and training solely with the GAN objective (denoted as ``Adversarial distillation" in the tables).

\begin{figure*}[t]
    \vspace{-0.7cm}
    \centering
  \begin{subfigure}[b]{0.3\textwidth}
    \includegraphics[width=\textwidth, trim=0.2cm 0.cm 0.cm 0.2cm, clip]{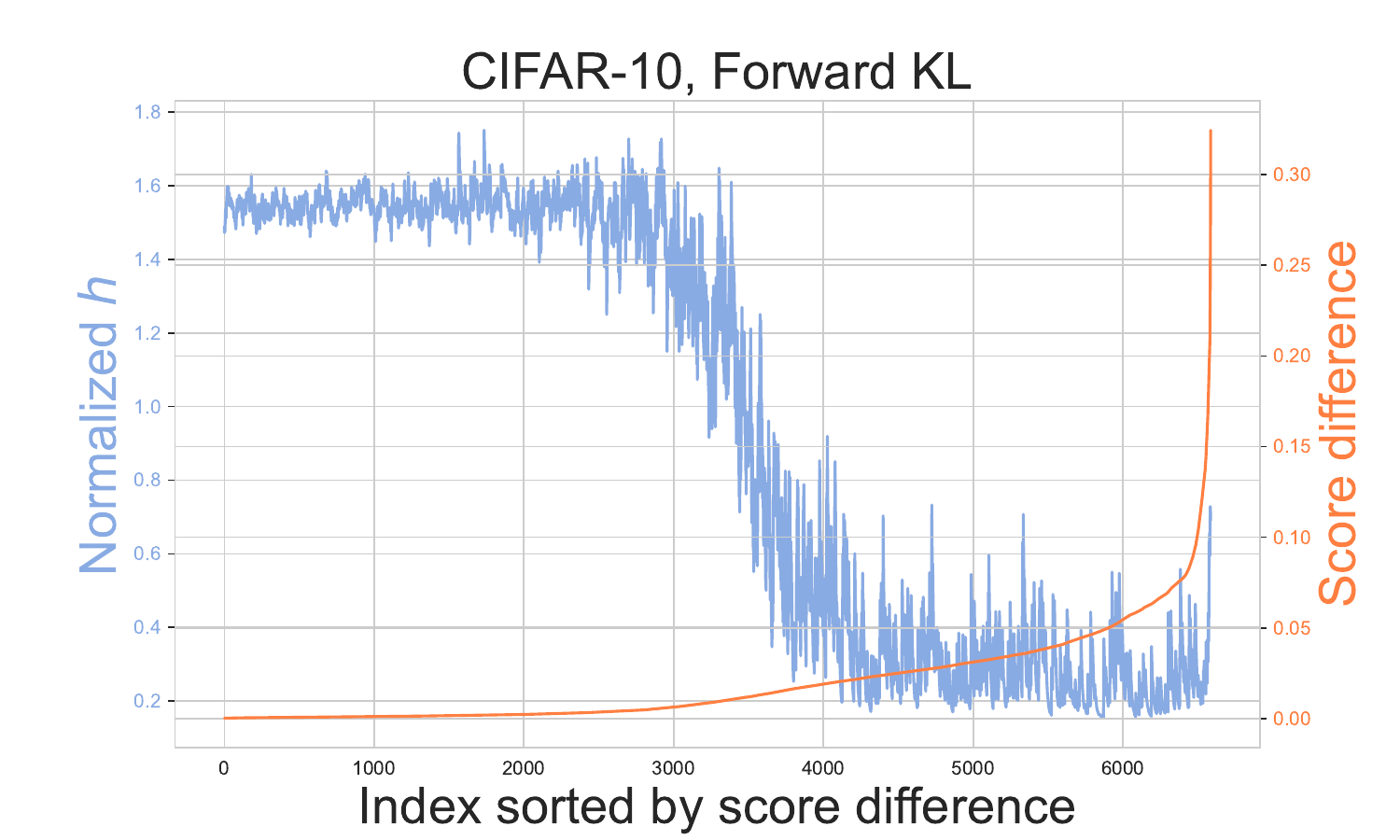}
        \vspace{-12pt}
    \caption{}
  \end{subfigure}
 \hfill
  \begin{subfigure}[b]{0.3\textwidth}
    \includegraphics[width=\textwidth, trim=0.2cm 0.cm 0.cm 0.2cm, clip]{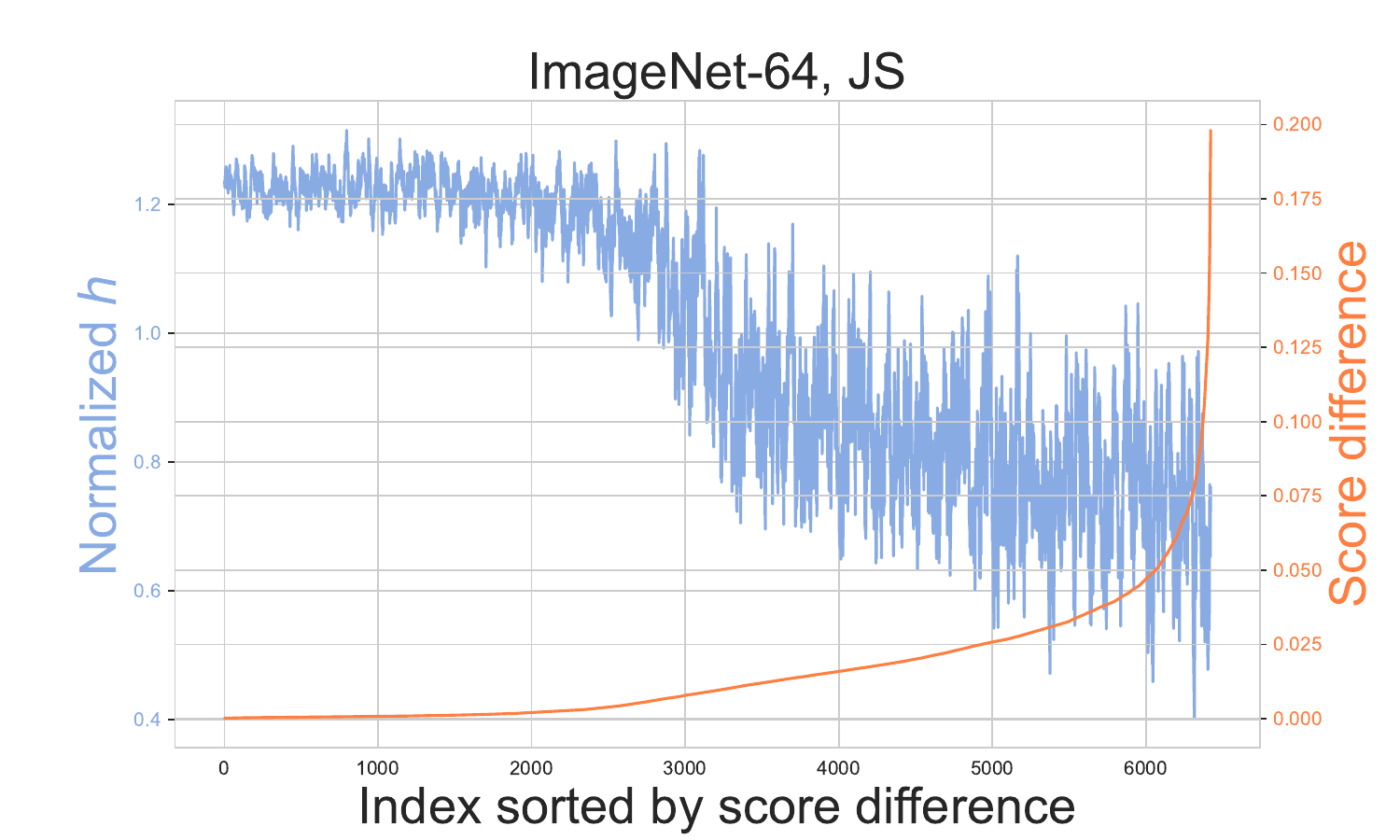}
    \vspace{-12pt}
    \caption{}
  \end{subfigure}\hfill
    \begin{subfigure}[b]{0.3\textwidth}
    \includegraphics[width=\textwidth, trim=0.2cm 0.cm 0.cm 0.2cm, clip]{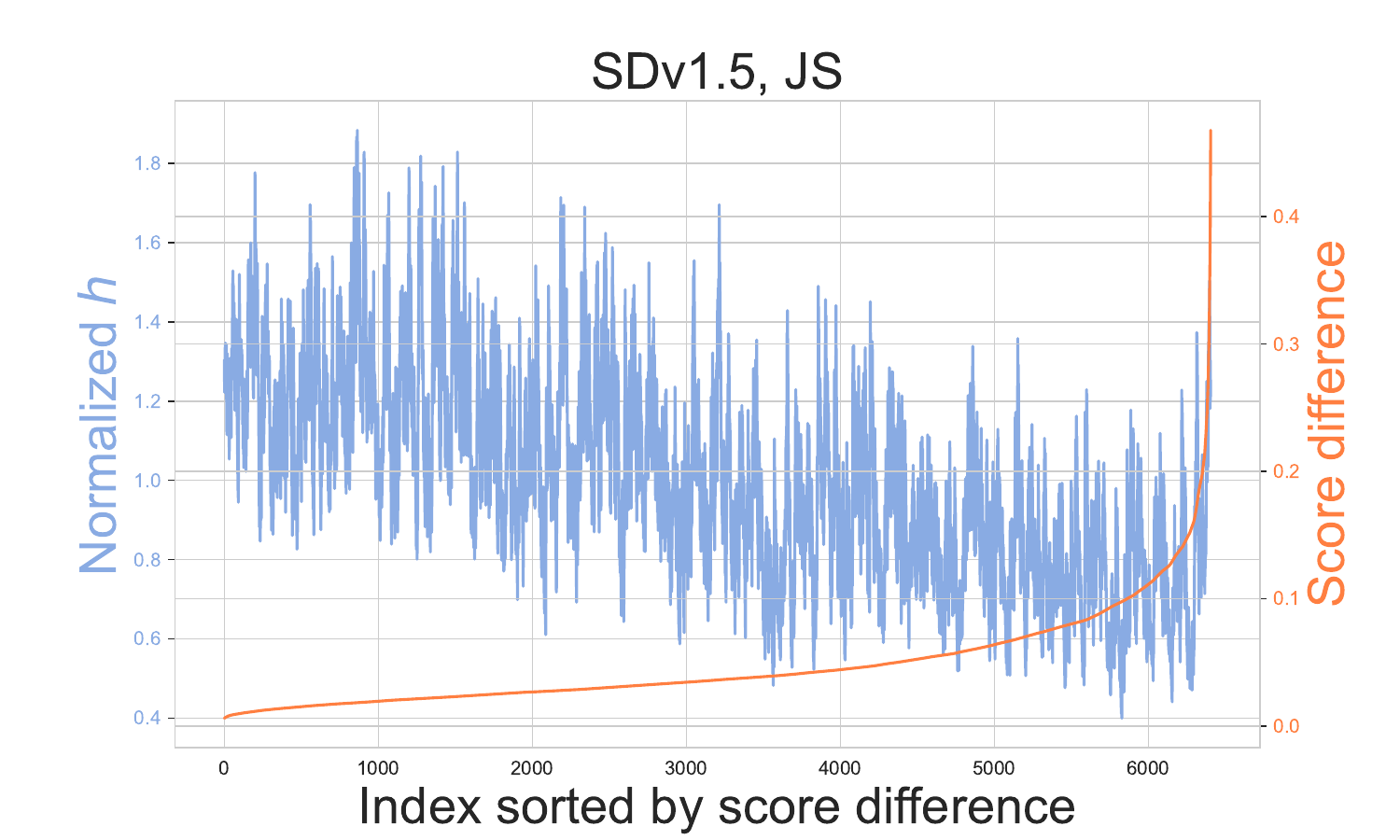}
    \vspace{-12pt}
    \caption{}
  \end{subfigure}
  \vspace{-6pt}
    \caption{Normalized weighting function $h$ and score difference versus the index of 6.4K generated samples, sorted by the score difference, on \textbf{(a)} CIFAR-10 with forward-KL; \textbf{(b)} ImageNet-64 with JS; \textbf{(c)} SDv1.5 with JS.}
    \label{fig:h-scorediff}
    \vspace{-6pt}
\end{figure*}

\begin{table}[t]
\scriptsize
    \centering
    \begin{tabular}{l c c c}
    \toprule
    & FID $\downarrow$ & CLIP score $\uparrow$& Latency $\downarrow$ \\
    \midrule
    \textbf{Multi-step diffusion models}\\
    LDM~\cite{rombach2021highresolution} & 12.63 && 3.7s \\
    DALL·E 2~\cite{ramesh2022dalle2} & 10.39 && 27s \\
Imagen~\cite{Saharia2022PhotorealisticTD} & 7.27 &0.28*& 9.1s \\
    eDiff-I~\cite{balaji2022ediffi} & \bf{6.95} &0.29*& 32.0s \\
    UniPC~\cite{zhao2024unipc} & 19.57 &&  0.26s \\
    Restart~\cite{Xu2023RestartSF} & 13.16 & 0.299 &3.40s\\
    & \\[-1.9ex]
    \cdashline{1-4}
    & \\[-1.9ex]
    \textbf{Teacher}\\
    SDv1.5 (NFE=100, CFG=3, ODE) & 8.59 & 0.308&  2.59s\\
    SDv1.5 (NFE=200, CFG=2, SDE) & 7.21 & 0.301 &10.25s \\
        \midrule
    \textbf{GANs}\\
    StyleGAN-T~\cite{sauer2023stylegan} & 13.90 &0.29*&  0.10s\\
    GigaGAN~\cite{kang2023scaling} & 9.09 &  &0.13s\\
    \midrule
    \textbf{Diffusion distillation}\\
    SwiftBrush~\cite{nguyen2024swiftbrush} & 16.67 &0.29* & 0.09s\\
    SwiftBrush v2~\cite{dao2025swiftbrush} & 8.14 & 0.32* & 0.06s\\
    HiPA~\cite{zhang2023hipa} & 13.91 && 0.09s\\
    InstaFlow-0.9B~\cite{liu2023instaflow} & 13.10 && 0.09s\\
    UFOGen~\cite{xu2024ufogen} & 12.78 && 0.09s\\
    DMD~\cite{yin2024one} & 11.49 && 0.09s \\
    EMD~\cite{xie2024distillation} & 9.66 && 0.09s \\
        & \\[-1.9ex]
    \cdashline{1-4}
    & \\[-1.9ex]
    \textit{CFG=1.75} \\
    reverse-KL (DMD2~\cite{yin2024improved}) & 8.17 & 0.287 & 0.09s\\
    JS~(\textit{ours}) & \bf{7.42} & 0.292 & 0.09s\\
        & \\[-1.9ex]
    \cdashline{1-4}
    & \\[-1.9ex]
    \textit{CFG=5} \\
    reverse-KL (DMD2~\cite{yin2024improved}) & 15.23 & 0.309 & 0.09s\\
    JS~(\textit{ours}) & 14.25 & 0.311 & 0.09s\\
    \bottomrule
    \end{tabular}
    \caption{FID score together with inference latency on text-to-image generation, zero-shot MS COCO-30k 512$\times $ 512. * denotes that the value is taken from the corresponding paper.}
    \label{tab:sd}
    \vspace{-5pt}
\end{table}


    

\vspace{-10pt}
\paragraph{Evaluations.} We measure sample quality with Fr\'{e}chet Inception Distance (FID)~\citep{heusel2017gans}. For diversity, we use the Recall score~\cite{Kynknniemi2019ImprovedPA}. For image-caption alignment, we report the CLIP score. We defer more evaluations to App~\ref{app:diversity} and ~\ref{app:hpsv2} on other metrics, such as in-batch similarity~\cite{corso2023particle} and HPSv2 score~\cite{dao2025swiftbrush} for diversity and image quality assessment.

\vspace{-10pt}
\paragraph{Results.} We first experiment with various $f$-divergences on CIFAR-10, to analyze the relative performance of different $f$s. Table~\ref{tab:cifar10} shows that all the variants under \methodtext outperform the adversarial distillation baseline, validating the effectiveness of \methodtext in addition to GAN objective. In particular, $f$-divergences with milder mode-seeking behavior, such as forward-KL and Jeffreys divergences, generally yield lower FID scores and higher Recall scores.

\begin{figure*}[t]
\centering
\includegraphics[width=0.91\textwidth]{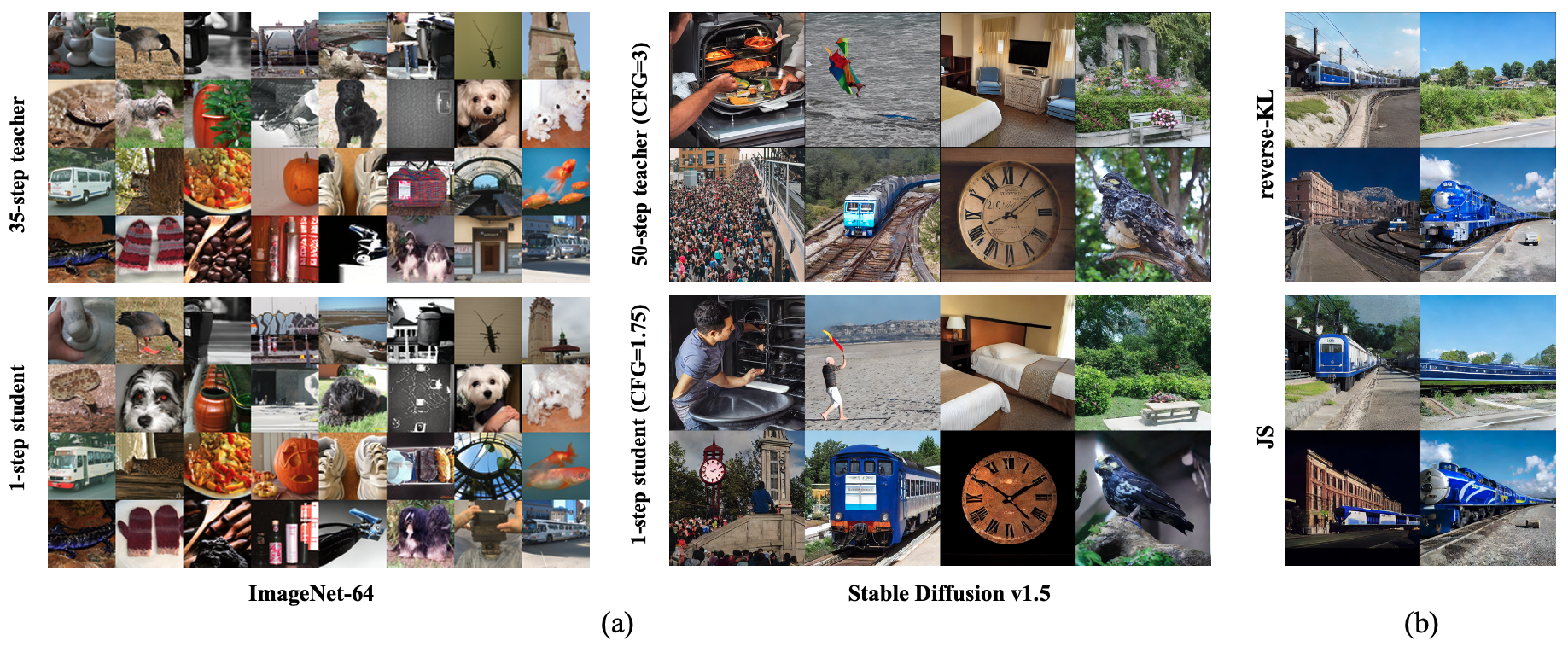}
\vspace{-11pt}
\caption{\textbf{(a)} Uncurated generated samples by the multi-step teacher diffusion models (top), and one-step student in \methodtext (bottom), using the same random seed. The teacher diffusion models use 35 and 50 steps on ImageNet-64 and Stable Diffusion v1.5, respectively. \textbf{(b)} Generated samples by reverse-KL and JS,  using a prompt in COYO: ``\texttt{a blue and white passenger train coming to a stop}".\looseness=-1}
\label{fig:vis}
\vspace{-10pt}
\end{figure*}

Table~\ref{tab:imagenet-64} and Table~\ref{tab:sd} report FID, Recall and CLIP score in two more challenging datasets. We report the inference latency on a single NVIDIA A100 GPU on text-to-image generation, as in \cite{yin2024improved}. Our main findings are: (1) \textbf{\methodtext with JS divergence achieves the current state-of-the-art one-step FID score on both ImageNet-64 and zero-shot MS COCO.} Concretely, JS achieves FID scores of 1.16 on ImageNet-64, outperforming previous best-performing diffusion models, GANs, and distillation methods, except for GDD-I~\cite{zheng2024diffusion}. GDD-I solely applies the GAN objective, which is known to be unstable in large-scale settings. It is reflected in the worse Recall score of GDD-I. Furthermore, JS obtains an FID score of {\bf{7.42}} on MS COCO, when using a CFG=1.75, significantly outperforming previous distillation methods and approaching the performance of leading diffusion models like eDiff-I~\cite{balaji2022ediffi}~(FID of 6.95). (2) \textbf{Forward-KL and JS get better FID scores than reverse-KL}. The two variants with less mode-seeking behavior continue to outperform the reverse-KL~(DMD2~\cite{yin2024improved}). 
When using a higher CFG value 5  in distillation, the JS outperforms most of all the baselines except SwiftBruch v2~\cite{dao2025swiftbrush}, which uses a more advanced SD2.1 as teacher model and additional CLIP loss during training.
(3) \textbf{JS is more stable than forward-KL}. While forward-KL can be less mode-seeking, it suffers from higher variance.  Our experiments on ImageNet-64 confirm that JS exhibits a significantly more stable training process~(App~\ref{app:results}), resulting in a better FID score. We further scale JS to larger models SDXL and observed competitive results to teacher, as shown in Fig.~\ref{fig:vis-main}. \looseness=-1

We further provide a visual comparison of generated samples by teacher and student in Fig.~\ref{fig:vis}. We observe that the generated images by one-step students generally have richer details and more aligned with the text prompts. We provide detailed prompts and extended samples in App~\ref{app:extended} in the supplementary material.

\subsection{Behavior of weighting function $h$}

As discussed in Section~\ref{sec:properties}, $f$-divergence with less mode-seeking tendancy has a faster increasing second derivative $f''$, resulting in an increasing weighting function $h(r)=f''(r)r^2$. This means that generated samples in the low-density regions of the data distribution $p$ will be down-weighted accordingly, and the teacher models are prone to inaccurate score estimation in these regions~\cite{karras2024guiding}.

To further understand the behavior of $h$, we study its relation with the score difference between a teacher and fake score, \ie $||s_\phi(\rvx; \sigma(t)) - s_\psi(\rvx, \sigma(t))||_2$ on real datasets. Recall that the teacher $s_\phi$ approximates the true score, \ie $s_\phi(\rvx; \sigma(t)) \approx \nabla_\rvx \log p_t(\rvx)$, and the online fake score approximate the generated distribution, \ie $s_\psi (\rvx, \sigma(t)) \approx \nabla_\rvx \log q_t(\rvx)$. We compute both $h$ and the score difference for 6.4K generated samples and sort them in ascending order of their score difference. Fig.~\ref{fig:h-scorediff} shows that the sample's weighting $h$ generally goes in the opposite direction with its score difference when using less mode-seeking divergences. This observation suggests that when using non-mode-seeking divergences~(\textit{e.g.,} forward-KL, JS), \methodtext effectively downweights samples in regions where the teacher and fake scores exhibit substantial discrepancies, which typically correspond to low-density regions of the true data distribution.

\section{Related work}
\vspace{-3pt}

As the sampling process in diffusion models is essentially solving the ODEs or SDEs~\citep{song2020score}, many early works focus on reducing the sampling steps with faster numerical solvers~\citep{Song2020DenoisingDI,lu2022dpmp,Karras2022ElucidatingTD,liu2022pseudo,zheng2024dpm}. However, they usually still require more than 20 steps due to the discretization error. Diffusion distillation has recently attracted more attention due to its promising goal of reducing the number of sampling steps to one single network call. 
It mainly includes two classes of distillation approaches: 

(1) \emph{Trajectory distillation}, which trains a one-step student model to mimic the deterministic sampling process of the teacher diffusion model. Knowledge distillation~\citep{luhman2021knowledge,zheng2022fast} learns a direct mapping from noise to data. Progressive distillation~\cite{salimans2022progressive,meng2023distillation} iteratively halves the number of sampling steps via distillation. Consistency models~\citep{song2023consistency,song2023improved,geng2024consistency,lee2024truncated,lu2024simplifying} lean a consistency function that maps any noisy data along an ODE trajectory to the associated clean data. 

(2) \emph{Distribution matching}, which aligns the distribution of the one-step student with that of the teacher diffusion model. 
Adversarial distillation~\citep{sauer2023adversarial,sauer2024fast,xu2024ufogen} mainly relies on the adversarial training~\citep{goodfellow2014generative} to learn teacher output's distribution. Another line of approaches implicitly minimizes various divergences, often via variational score distillation~\citep{wang2024prolificdreamer}, such as reverse-KL~\citep{yin2024one,yin2024improved}, forward KL~\citep{luo2024diff,xie2024distillation} and fisher divergence~\citep{zhou2024score}. Score Implicit Matching~\citep{luo2024one} generalizes the fisher divergence~\citep{zhou2024score} by relaxing the score-based distance to have more general forms beyond squared $L_2$. Our method lies in this category. Different from previous methods that only minimize a particular distribution divergence, each of which may require vastly different training strategies~\citep{yin2024improved,xie2024distillation,zhou2024score}, our method unifies the class of $f$-divergences in a principled way and thus offers better flexibility in distribution matching distillation.

\vspace{-2pt}
\section{Conclusions}
\vspace{-2pt}
We have proposed \method, a novel and general framework for distributional matching distillation based on $f$-divergence minimization. We derive a gradient update rule comprising the product of a weighting function and the score difference between the teacher and student distributions. \methodtext encompasses previous variational score distillation objectives while allowing less mode-seeking divergences. By leveraging the weighting function, \methodtext naturally downweights regions with larger score estimation errors.  Experiments on various image generation tasks demonstrate the strong one-step generation capabilities of \method.

{
    \small
    \bibliographystyle{ieeenat_fullname}
    \bibliography{main}
}

\clearpage

\onecolumn
{

    \centering
    \Large
    \textbf{One-step Diffusion Models with $f$-Divergence Distribution Matching}\\
    \vspace{0.5em}Supplementary Material \\
    \vspace{1.0em}
}

\appendix

\section{Proofs}
\label{app:proof}

In this section, we provide proofs for Theorem~\ref{thm-main} and Proposition~\ref{prop:general} in the main text. We will start with Lemma~\ref{lemma:cov} before proving Theorem~\ref{thm-main}.

\setcounter{theorem}{0}
\setcounter{proposition}{1}

\begin{lemma}
\label{lemma:cov}
Assuming that sampling from $\rvx \sim  q_t(\rvx)$ can be parameterized to $\rvx = G_\theta(\rvz)+\sigma(t)\epsilon$ for $\rvz \sim  p(\rvz), \epsilon \sim \mathcal{N}(  \mathbf{0}, \mI)$ and $G_\theta,g $ are differentiable mappings. In addition, $g$ is constant with respect to $\theta$. Then $\int \nabla_\theta  q_t(\rvx) g(\rvx) d\rvx = \int\int p(\epsilon)  p(\rvz) \nabla_\rvx g(\rvx) \nabla_\theta G_\theta(\rvz) d\epsilon d\rvz$.
\end{lemma}
\begin{proof}
    As $q_t$ and $g$ are both continuous functions, we can interchange integration and differentiation:
    \begin{align*}
        \int \nabla_\theta  q_t(\rvx) g(\rvx) d\rvx  &= \nabla_\theta \int   q_t(\rvx) g(\rvx) d\rvx \\
        &= \nabla_\theta \int  \int p(\epsilon)  p(\rvz)  g(G_\theta(\rvz)+\sigma(t)\epsilon)  d\epsilon d\rvz \\
        &=  \int  \int p(\epsilon)  p(\rvz)  \nabla_\theta g(G_\theta(\rvz)+\sigma(t)\epsilon)  d\epsilon d\rvz  \numberthis\label{eq:lemma1-inter}\\
        &=  \int  \int p(\epsilon)  p(\rvz)  \nabla_{\rvx} g(G_\theta(\rvz)+\sigma(t)\epsilon) \nabla_\theta G_\theta(\rvz)  d\epsilon d\rvz \\
        &= \int\int p(\epsilon)  p(\rvz) \nabla_\rvx g(\rvx) \nabla_\theta G_\theta(\rvz) d\epsilon d\rvz
    \end{align*}
where $\rvx = G_\theta(\rvz)+\sigma(t)\epsilon$. We can interchange integration and differentiation again in \Eqref{eq:lemma1-inter} as $g$ is a differentiable function.
\end{proof}
\begin{theorem}

Let $p$ be the teacher's generative distribution, and let $q$ be a distribution induced by transforming a prior distribution $ p(\rvz)$ through the differentiable mapping $G_\theta$. Assuming $f$ is twice continuously differentiable, then the gradient of $f$-divergence between the two intermediate distribution $p_t$ and $q_t$ w.r.t $\theta$ is:
{
\begin{align*} 
    &\nabla_\theta D_f(p_t||q_t) = \E_{\substack{\rvz,  \epsilon}}-\left[f''\left(\frac{p_t(\rvx)}{q_t(\rvx)}\right)\left(\frac{p_t(\rvx)}{q_t(\rvx)}\right)^2\left(\underbrace{\nabla_\rvx \log p_t(\rvx)}_{\textrm{teacher score}} - \underbrace{\nabla_\rvx \log q_t(\rvx)}_{\textrm{fake score}} \right)  \nabla_\theta G_\theta(\rvz)\right]
    \numberthis \label{eq:time-0-loss-2}
\end{align*}
}%
where $\rvz \sim  p(\rvz), \epsilon \sim \mathcal{N}(  \mathbf{0}, \mI)$ and $ \rvx = G_\theta(\rvz)+\sigma(t)\epsilon $
\end{theorem}

\begin{proof}

Note that both the intermediate student distribution $q_t$ and the sample $\rvx$ have a dependency on the generator parameter $\theta$. In the proof, we simplify the expression $\int (\nabla_\theta  q_t(\rvx)) g(\rvx) d\rvx$ as $\int \nabla_\theta  q_t(\rvx) g(\rvx) d\rvx$ for clarity. The total derivative of $f$-divergence between teacher's and student's intermediate distribution is as follows:

\begin{align*}
&\nabla_\theta D_f(p_t(\rvx)||q_t(\rvx)) =   \nabla_\theta \int q_t(\rvx) f(\frac{p_t(\rvx)}{q_t(\rvx)}) d\rvx \\
&=   \int \nabla_\theta q_t(\rvx) f(\frac{p_t(\rvx)}{q_t(\rvx)}) d\rvx + \int  q_t(\rvx) \nabla_\theta f(\frac{p_t(\rvx)}{q_t(\rvx)}) d\rvx \\
&=  \int \nabla_\theta q_t(\rvx) f(\frac{p_t(\rvx)}{q_t(\rvx)}) d\rvx - \int  q_t(\rvx)  f'(\frac{p_t(\rvx)}{q_t(\rvx)})\frac{p_t(\rvx)}{q^2_t(\rvx)} \nabla_\theta q_t(\rvx)d\rvx\\
    &=  \underbrace{\int \nabla_\theta q_t(\rvx) f(\frac{p_t(\rvx)}{q_t(\rvx)}) d\rvx }_{I} - \underbrace{\int \nabla_\theta q_t(\rvx) f'(\frac{p_t(\rvx)}{q_t(\rvx)}) \frac{p_t(\rvx)}{q_t(\rvx)}  d\rvx}_{II} \numberthis \label{eq:combined}
\end{align*}

Note that by notation $\nabla_\theta  q_t(\rvx)$, we mean that only the first $q$ inisde each integral has gradient w.r.t $\theta$. (I) / (II) is the term associated with the partial derivative of $f$ with respect to $\rvx$ / $q$, respectively. Above we see that both partial derivatives (I) and (II) are in the form $\int \nabla_\theta  q_t(\rvx) g(\rvx) d\rvx$ where $g$ is a differentiable function that is constant with respect to $\theta$. Using the identity in Lemma~\ref{lemma:cov} again, we can simplify (I) and (II) to:
{
\begin{align*}
I &= \int \int p(\epsilon)  p(\rvz) f'(\frac{ p_t(\rvx)}{ q_t(\rvx)}) \nabla_\rvx \frac{ p_t(\rvx)}{ q_t(\rvx)} \nabla_\theta G_\theta(\rvz) d\epsilon d\rvz \\
II &= \int \int p(\epsilon) p(\rvz) f''(\frac{ p_t(\rvx)}{ q_t(\rvx)}) \frac{ p_t(\rvx)}{ q_t(\rvx)} \nabla_\rvx \frac{ p_t(\rvx)}{ q_t(\rvx)} \nabla_\theta G_\theta(\rvz)d\epsilon d\rvz + \int \int p(\epsilon) p(\rvz) f'(\frac{ p_t(\rvx)}{ q_t(\rvx)}) \nabla_\rvx \frac{ p_t(\rvx)}{ q_t(\rvx)} \nabla_\theta G_\theta(\rvz) d\epsilon d\rvz 
\end{align*}
}%
Putting (I) and (II) in \Eqref{eq:combined}, we have:
{
\begin{align*}
    \nabla_\theta D_f(p_t(\rvx)||q_t(\rvx)) = &-\int  \int p(\epsilon) p(\rvz) f''(\frac{ p_t(\rvx)}{ q_t(\rvx)}) \frac{ p_t(\rvx)}{ q_t(\rvx)} \nabla_\rvx \frac{ p_t(\rvx)}{ q_t(\rvx)} \nabla_\theta G_\theta(\rvz) d\epsilon d\rvz \\
= &-\int  \int p(\epsilon) p(\rvz) f''\left(\frac{ p_t(\rvx)}{ q_t(\rvx)}\right)\left(\frac{ p_t(\rvx)}{ q_t(\rvx)}\right)^2\left[\nabla_\rvx \log  p_t(\rvx) - \nabla_\rvx \log  q_t(\rvx)\right] \nabla_\theta G_\theta(\rvz) d\epsilon d\rvz \\
= &\E_{\substack{\rvz,  \epsilon}}-\left[f''\left(\frac{p_t(\rvx)}{q_t(\rvx)}\right)\left(\frac{p_t(\rvx)}{q_t(\rvx)}\right)^2\left(\underbrace{\nabla_\rvx \log p_t(\rvx)}_{\textrm{teacher score}} - \underbrace{\nabla_\rvx \log q_t(\rvx)}_{\textrm{fake score}} \right)  \nabla_\theta G_\theta(\rvz)\right]
\end{align*}
}%
where the last identity is from the log derivative trick, \ie $\nabla_\rvx \frac{p_t(\rvx)}{q_t(\rvx)} = \frac{p_t(\rvx)}{q_t(\rvx)}\left[\nabla_\rvx \log  p_t(\rvx) - \nabla_\rvx \log  q_t(\rvx)\right] $. 

\end{proof}

\mainprop*

\begin{proof}
    To constitute a valid $f$-divergence, the requirement for $f $ is that  $f$ is a convex function on $(0,+\infty)$ satisfying $f(1)=0$. For any function $h$ that is continuous and non-negative function on $(0, +\infty)$, the function $g(r) = h(r)/r^2$ is also a continuous and non-negative function on $(0, +\infty)$. By the fundamental theorem of calculus, we know that there exists a continuous function $m(r)$ whose second derivative equals to $g(r)$, \ie $m''(r) = g(r)$. Let $f(r) = m(r) - m(1)$, it is straightforward to see that $f(1)=0$ and $f''(r)=h(r)/r^2$. In addition, $f$ is a convex function on $(0, +\infty)$ as its second derivative is non-negative in this domain. Let $ \rvx = G_\theta(\rvz)+\sigma(t)\epsilon $, we can re-express the expectation as follows:
    \begin{align*}
        &\E_{\substack{\rvz,  \epsilon}}-\left[h\left(r_t(\rvx)\right)\right.\left({\nabla_\rvx \log p_t(\rvx)} - {\nabla_\rvx \log q_t(\rvx)} \right)  \nabla_\theta G_\theta(\rvz)] \\
        &= \E_{\substack{\rvz,  \epsilon}}-\left[f''\left(r_t(\rvx)\right)r^2_t(\rvx)\right.\left({\nabla_\rvx \log p_t(\rvx)} - {\nabla_\rvx \log q_t(\rvx)} \right)  \nabla_\theta G_\theta(\rvz)] \\
        &= \nabla_\theta D_f(p_t(\rvx)||q_t(\rvx)) 
    \end{align*}
    where the last equation is by Theorem~\ref{thm-main}.
\end{proof}

\section{Training details}

\label{app:details}

In this section, we provide training details for \methodtext on CIFAR-10, ImageNet-64, COYO-700M~(w/ SD v1.5 model) and COYO-700M-Aesthetic-6.5~(w/ SDXL model). The COYO-700M-Aesthetic-6.5 is a subset of COYO-700M by filtering out images with aesthetic score lower than 6.5. Table~\ref{tab:training} shows the values of common training hyper-parameters on different datasets. For most hyper-parameters, we directly borrow the value from \cite{yin2024improved}, which is a special case in the \methodtext framework. Inspired by the three-stage training in  \cite{yin2024improved}, we also divide the ImageNet-64 training process into two stages with different learning rates. In the first stage, we train the model with a learning rate of 2e-6 for 200k iterations, then fine-tune it with a learning rate of 5e-7 for 180k iterations. We apply TTUR~\cite{yin2024improved} for all the models. We further provide an algorithm box in Alg~\ref{alg:f-distill} for clarity. For hyper-parameters, we use a batch size of 2048 / 512 / 1024 / 384 for CIFAR-10 / ImageNet-64 / COYO (SD v1.5) / COYO (SDXL).

\cite{yin2024improved} uses the online fake score network as the feature extractor for the GAN discriminator. This complicates the training process, as there is an additional hyper-parameter balancing the denoising score-matching loss and GAN loss for updating the fake score network. To simplify the use of GAN in our framework, we use the fixed teacher network as the feature extractor, similar to LADD~\cite{sauer2024fast}. Unlike \cite{yin2024improved}, including the fake score network as part of the learnable parameter in the GAN discriminator, the learnable parameter in the new setup is a small classification head whose input is the feature from the teacher network. We empirically observe that the modification leads to better performance on CIFAR-10, as shown in Fig.~\ref{fig:feature-a}. We also include an ablation study~(green line in Fig.~\ref{fig:feature-b}), which uses the fake score as a feature extractor but does not update it in the GAN loss. The model behaves poorly in this case, as the fake score is constantly getting updated with denoising score-matching loss, validating the benefits of using the fixed teacher score as feature extractor.

For evaluation, we report the FID / Recall score on 50K samples on CIFAR-10 and ImageNet-64, and 30K samples on zero-shot MS-COCO dataset. We report the CLIP score on 30K samples using MS-COCO datasets.

\begin{table}[t]
\footnotesize
    \centering
    \begin{tabular}{l c c c c}
    \toprule
     & CIFAR-10 & ImageNet-64 & COYO-700M & COYO-700M-Aesthetic-6.5  \\
    \midrule
    Batch size & 2048 & 512 & 1024 & 384 \\
    Fake score update frequency & 5 & 5 &  10 & 5\\
    GAN loss weight & 1e-3 & 3e-3 & 1e-3& 5e-3\\
    GAN discriminator input resolution & (32, 16, 8) & (8) & (8) & (16)\\
    Total iteration & 60K & 380K & 60K & 25K\\
    Teacher & EDM~\cite{Karras2022ElucidatingTD} & EDM~\cite{Karras2022ElucidatingTD} & Stable Diffusion v1.5~\cite{rombach2022high} & SDXL~\cite{podell2024sdxl}\\
    CFG weight & 1 & 1 & 1.75/5 & 8 \\
        & \\[-1.9ex]
    \cdashline{1-5}
    & \\[-1.9ex]
    \textbf{Adam optimizer} \\
    Learning rate & 1e-4 & 2e-6 &  1e-5&  5e-7\\
    Learning rate for fine-tuning & - & 5e-7 & - & - \\
    Weight decay & 1e-2 & 1e-2 & 1e-2 & 1e-2  \\
    $\gamma$ in R1 regularization & 1 & 0 & 1 & 1\\
    \bottomrule
    \end{tabular}
    \caption{Training configuration for \methodtext on different datasets.}
    \label{tab:training}
\end{table}

\begin{figure*}[t]
    \centering
  \begin{subfigure}[b]{0.35\textwidth}
    \includegraphics[width=\textwidth]{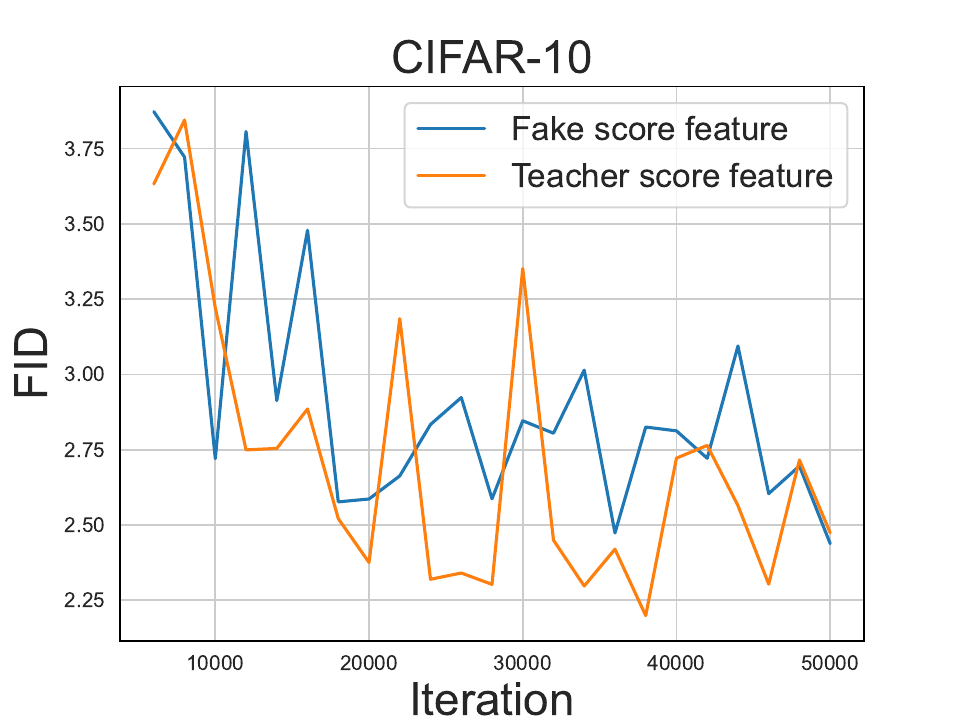}
        \vspace{-3pt}
    \caption{}
        \label{fig:feature-a}
  \end{subfigure}
 \hspace{3pt}
  \begin{subfigure}[b]{0.35\textwidth}
    \includegraphics[width=\textwidth]{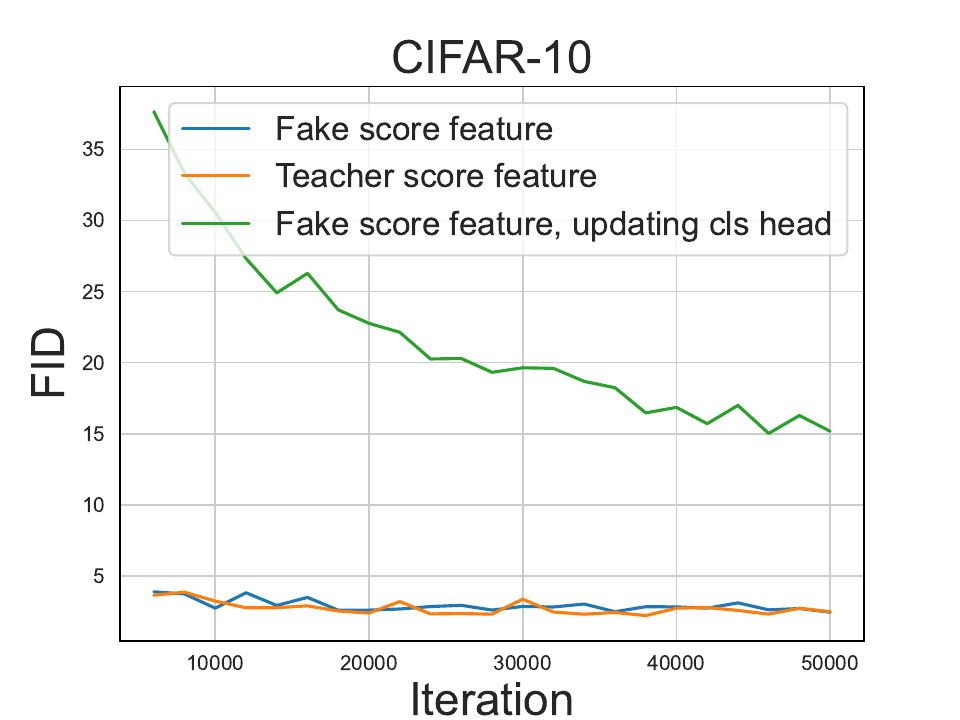}
    \vspace{-3pt}
    \caption{}
    \label{fig:feature-b}
  \end{subfigure}
    \caption{FID score versus training iteration on CIFAR-10. \textit{Fake score feature}: fake score as the extractor, updating both the fake score and classification head in the GAN discriminator loss. \textit{Teacher score feature}: teacher score as the extractor, updating classification head in the GAN discriminator loss. \textit{Teacher score feature, updating cls head}: fake score as the extractor, updating classification head in the GAN discriminator loss. (a) is the zoomed-in visualization of (b).
    }
    \label{fig:feature}
\end{figure*}

\begin{algorithm}[t]
\small
    \caption{$f$-distill Training}
    \label{alg:f-distill}
    \begin{algorithmic}[1]
        \State \textbf{Input:} Teacher diffusion model $s_\phi$, fake score $s_\psi$, one-step student $G_\theta$, discriminator $D_\lambda$, total iteration $M$, batch size $B$, fake score update frequency $\tau$, weighting function $h$, GAN loss coefficient $w_{\textrm{GAN}}$, minimun/maximum value for ratio $r_{\textrm{min}}$/$r_{\textrm{max}}$.
        \For{iteration in $0 \dots M$}
                    \State $\rvx_1,...,\rvx_B \sim p_{\text{data}}$, $\epsilon_1,...,\epsilon_B \sim \mathcal{N}(  \mathbf{0}, \mI)$,
            $t_{1},...,t_B \sim \gU\{1, \dots, T\}$
            \State Generate the data by student: $\rvy_i = G_\theta(\epsilon_i)$
        \If{iteration$\% \tau = 0$}  \Comment{Update student}
            \State Compute and clip the density ratio: $r_i = \texttt{clip}(\exp(D_\lambda(\rvy_i)),r_{\textrm{min}},r_{\textrm{max}}) $
               \State Normalize the weighting coefficient: $\bar{r}_i = r_i / \sum_i r_i$ \Comment{First-stage normalization}
            \State Compute the weighting coefficient: $h_i = h(\bar{r}_i)$
            \State Normalize the weighting coefficient: $\bar{h}_i = h_i / \frac{1}{B}\sum_i h_i$ \Comment{Second-stage normalization}
            \State Compute the empirical $f$-distill loss $\tilde{\gL}_{\textrm{\method}}(\theta)$ based on \Eqref{eq:obj-final} with $(\rvy_i, t_i, \bar{h}_i)_{i=1}^B$ and fake score $s_\psi$
            \State Compute the empirical GAN loss $\tilde{\gL}_{\textrm{GAN}}(\theta)$ for generator with $(\rvy_i, t_i)_{i=1}^B$ 
            \State Update the student parameter $\theta$ by $\tilde{\gL}_{\textrm{\method}}(\theta)+w_{\textrm{GAN}}\tilde{\gL}_{\textrm{GAN}}(\theta)$.
        \Else \Comment{Update fake score and discriminator}
            \State Update the fake score with denoising score-matching loss using $(\rvy_i, t_i)_{i=1}^B$.
            \State Update the discriminator with empirical GAN loss $\tilde{\gL}_{\textrm{GAN}}(\lambda)$ for discriminator with $(\rvx_i, \rvy_i, t_i)_{i=1}^B$ 
        \EndIf
        
        \EndFor
    \end{algorithmic}
    \end{algorithm}

\section{Properties of $f$-divergence}

\label{app:property}

\subsection{Mode-seeking behavior in $f$-divergence}

\subsubsection{Classification by mode-seeking}
Mode-seeking, as described in Section 10.1.2 in \cite{Bishop2006PatternRA}, refers to the tendency of fitted generative models to capture only a subset of the dominant modes in the data distribution. This occurs during the minimization of the $f$-divergence $\min_q D_f(p||q)$ between the true data distribution ($p$) and the learned generative distribution ($q$).  An $f$-divergence is considered ``mode-seeking" if its minimization leads to this mode-seeking behavior in the corresponding generative model. The mode-seeking behavior in generative models translates into a lack of diversity in practice. Most of the previous classifications of mode-seeking divergences are mainly based on empirical observations. For example, reverse-KL is widely considered mode-seeking, and forward-KL aims for the opposite~(\ie mode-coverage)~\cite{Poole2016ImprovedGO}. Here, we applied the criteria proposed in \cite{pmlr-v206-ting-li23a}~(see Definition 4.1 in the paper) to roughly classify the $f$-divergence based on mode-seeking. Intuitively, a smaller limit indicates a higher tolerance of the corresponding $f$-divergence for large density ratios ($r=p/q$). This allows the generative distribution $q$ to assign less probability mass to regions where the true distribution $p$ has high density without incurring a significant penalty. Consequently, this behavior can lead to mode-seeking, where the model focuses on capturing only the dominant modes of the data distribution. Hence, we use the rate of the limit to classify divergence in the mode-seeking column in Table~\ref{tab:f-div}.

\subsubsection{Relation to the weighting function $h$}
Another paper~\cite{Shannon2020NonsaturatingGT} classifies the mode-seeking divergence based on the increasing rate of the limits $ \lim_{r\to \infty}f''(r)$ and $ \lim_{r\to 0}f''(r)$, through a concept of ``tail weight". The tail weight associated with $ \lim_{r\to \infty}f''(r)$~(right tail weight) / $ \lim_{r\to 0}f''(r)$~(left tail weight) describes how strongly the mode-seeking / mode-coverage behavior is penalized. A larger rate of limit can be translated into a higher penalty on mode-seeking. Table~\ref{tab:tail-weight-rate} demonstrates tail weights for different divergences. In general, less mode-seeking divergence will have a larger right tail weight~(rate of $\lim_{r\to \infty}f''(r)$) and smaller left tail weight~(rate of $\lim_{r\to 0}f''(r)$). As a result, when using these canonical $f$-divergences, the weighting function $h(r)$ would be an increasing function if the divergence is less mode-seeking since $h(r)= f''(r)r^2$. For example, $h$ in JS and forward-KL is an increasing function, while in reverse-KL, $h$ stays constant. An increasing h tends to downweight regions with lower density in the true data distribution $p$. This corresponds to regions where the teacher score is less reliable. Fig.~\ref{fig:teasdr-2} illustrates the idea.

\begin{table}[h]
    \centering
    \begin{tabular}{l c c c c c c}
    \toprule
     & reverse-KL & softened RKL & JS & squared Hellinger & forward-KL & Jefferys  \\
    \midrule
    Rate of $ \lim_{r\to \infty}f''(r)$ & $\gO(r^{-2})$& $\gO(r^{-3})$ & $\gO(r^{-2})$ & $\gO(r^{-\frac{3}{2}})$ & $\gO(r^{-1})$ & $\gO(r^{-1})$\\
    Rate of $ \lim_{r\to 0}f''(r)$& $\gO(r^{-2})$& $\gO(r^{-2})$ & $\gO(r^{-1})$ & $\gO(r^{-\frac{3}{2}})$ & $\gO(r^{-1})$ & $\gO(r^{-2})$\\
    \bottomrule
    \end{tabular}
    \caption{Right / left weight for different $f$-divergences. We shift the tail weight in \cite{Shannon2020NonsaturatingGT} by a constant for clarity.}
    \label{tab:tail-weight-rate}
\end{table}

\begin{figure}
    \centering
    \includegraphics[width=1\linewidth]{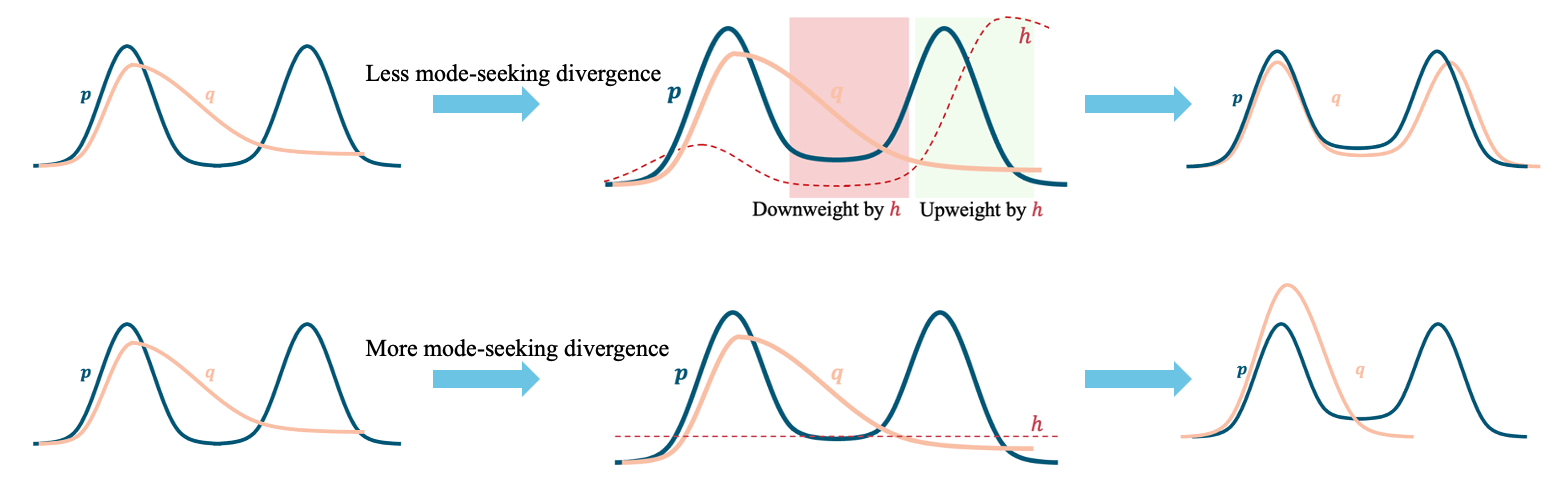}
    \caption{Illustration of how the weighting function $h$~(\textcolor{myred}{red dotted line}) in less mode-seeking divergence~(forward-KL, $h = p/q$) helps to learn the true data distribution $ p$, compared to more mode-seeking divergence~(reverse-KL, $h \equiv 1$). We illustrate how using a less mode-seeking divergence can better capture different modes, from a skewed initial generative distribution $q$, with the help of the weighting function.}
    \label{fig:teasdr-2}
\end{figure}

\subsection{$f$-divergence and Fisher divergence}

A line of work focuses on achieving distributional matching by minimizing the Fisher divergence or its variants~\cite{zhou2024score, luo2024one}. While $f$-divergence-based distillation methods match the probability density functions $p_t$~(teacher distribution) and $q_t$~(student distribution), Fisher divergence-based distillation aims to match the distributions by minimizing the distance between their score functions. This equates to matching the gradients of the log probability density functions, $\nabla_\rvx \log p_t(\rvx)$ and $\nabla_\rvx \log q_t(\rvx)$:
\begin{align*}
    &\textbf{($f$-divergence):} \quad \min_{q_t} \int q_t(\rvx) f\left(\frac{p_t(\rvx)}{q_t(\rvx)}\right) d\rvx \\
    &\textbf{(General Fisher divergence):} \quad \min_{q_t} \int q_t(\rvx) \rvd(\nabla_\rvx \log p_t(\rvx), \nabla_\rvx \log q_t(\rvx)) d\rvx \numberthis \label{eq:fisher}
\end{align*}
where $\rvd$ is a scalar-valued proper distance function satisfying $\rvd(\rvx) \ge 0$ and $\rvd(\rvx)=0$ if and only if $\rvx=0$. When $\rvd$ is squared $\ell_2$ distance, \Eqref{eq:fisher} reduces to Fisher divergence. 

In practice, directly minimizing the Fisher divergence in Equation \ref{eq:fisher} is challenging. Existing works often rely on certain assumptions and approximations to make this optimization tractable. For example, \cite{luo2024one} imposes the stop gradient operation on the sampling distribution $q_t$~(first term in the integral in \Eqref{eq:fisher}). In addition, they typically use a Monte-Carlo sampler~\cite{zhou2024score}, derived from Tweedie’s Formula, to estimate an intractable term in the gradient.

\section{Additional Results}
\label{app:results}
\subsection{CIFAR-10 result with more baselines}

Due to space limits, we do not include baselines in the table for CIFAR-10~(Table~\ref{tab:cifar10}) in the main text. We provide a more comprehensive comparison in Table~\ref{tab:cifar10-more}. Please note that we mainly use this dataset for analyzing the difference of variants under the \methodtext family. Unlike previous works using pre-trained feature extractor~\cite{Kim2023ConsistencyTM, zheng2024diffusion}, we use a simple classification head as a discriminator on top of the teacher's features, and do not tune hyper-parameter on this dataset. Note that our approach simply employs the teacher model as feature extractor unlike GDD-I and CTM, as we primarily use this dataset to analyze the relative performance of different $f$s\looseness=-1.

\begin{table}[htbp]
\footnotesize
    \centering
    \begin{tabular}{l c c c}
    \toprule
    & FID $\downarrow$ & Recall $\uparrow$ & NFE\\
        \textbf{Multi-step diffusion models}\\
        \midrule
           DDPM~\citep{ho2020ddpm} & 3.17 & & 1000  \\    
    LSGM~\citep{vahdat2021score} & 2.10 & & 138 \\
  EDM~\cite{Karras2022ElucidatingTD}~(teacher) & 1.79 & & 35 \\
    PFGM++~\citep{Xu2023PFGMUT} & 1.74 & & 35 \\
    \midrule
    \textbf{GANs}\\
    BigGAN~\cite{Brock2018LargeSG} & 14.73& & 1\\ 
    StyleSAN-XL~\cite{Sauer2022StyleGANXLSS} & 1.36 & & 1\\
    \midrule
    \textbf{Diffusion distillation}\\
    Adversarial distillation & 2.60 & &1\\
    SiD~($\alpha=1$)~\cite{zhou2024score} & 1.93  & &1\\
    SiD~($\alpha=1.2$)~\cite{zhou2024score} & 1.71  & &1\\
    CTM~\cite{Kim2023ConsistencyTM} & 1.73  & &1\\ 
    GDD~\cite{zheng2024diffusion} & 1.66 & &1\\
    GDD-I~\cite{zheng2024diffusion} & 1.44 & &1\\
    \midrule
    \textbf{\method}\\
    reverse-KL~(\textcolor{red}{\cmark}, DMD2 \cite{yin2024improved}) & 2.13 & 0.60 &1\\
    softened RKL~(\textcolor{red}{\cmark}) & 2.21 & 0.60 &1\\
    squared Hellinger~(\textcolor{orange}{--}) & 1.99 & 0.63   &1\\
    JS~(\textcolor{orange}{--}) & 2.00 & 0.62 &1\\
    Jeffreys~(\textcolor{green}{\xmark}) & 2.05 &0.62 &1\\
    forward-KL~(\textcolor{green}{\xmark}) & 1.92 & 0.62 &1\\
    \bottomrule
    \end{tabular}
    \vspace{-5pt}
    \caption{FID and Recall scores on CIFAR-10. \textcolor{red}{\cmark}/\textcolor{orange}{--}/\textcolor{green}{\xmark} stand for high/medium/low mode-seeking tendency for $f$-divergence. }
    \label{tab:cifar10-more}
\end{table}

\subsection{Diversity evaluation based on in-batch similarity}
\label{app:diversity}
\begin{table}[t]
\footnotesize
    \centering
    \begin{tabular}{l c}
    \toprule
    &  In-batch-sim~$\downarrow$ \\
    \midrule
    
    Teacher (NFE=50, CFG=3, ODE)  & 0.55 / 0.70 \\
      Teacher (NFE=50, CFG=8, ODE)  & 0.62 / 0.72\\
      & \\[-1.9ex]
    \cdashline{1-2}
    & \\[-1.9ex]
    \textit{CFG=1.75} \\
    reverse-KL (DMD2)  & 0.50 / 0.42 \\
    JS~(\textit{ours})  & 0.49 / 0.41 \\
                & \\[-1.9ex]
    \cdashline{1-2}
    & \\[-1.9ex]
    \textit{CFG=5} \\
    reverse-KL (DMD2)  & 0.67 / 0.60\\
    JS~(\textit{ours})  & 0.65 / 0.58 \\
    \bottomrule
    \end{tabular}
    \caption{In-batch similarity on MS COCO 2014 / Parti-Prompt. We did not use Recall for text-to-image evaluation because it requires generating numerous samples per prompt. Furthermore, we found the Diversity score in \cite{yin2024improved} to be unreliable; higher CFG values in the teacher model, which are known to reduce diversity, will result in better Diversity scores.}
    \label{tab:sd-diversity}
\end{table}

It is important to understand the diversity of generated samples given a text prompt. We did not use Recall because it is unsuitable for measuring diversity in text-to-image generation, as it requires generating numerous samples per prompt. Furthermore, we found the Diversity score in \cite{yin2024improved} unreliable: higher CFG values in the teacher model, known to reduce diversity, will result in better Diversity scores. Hence, we opt for the in-batch similarity score. As a result, we use the in-batch similarity~\cite{corso2023particle} to measure the diversity. In-batch similarity~\cite{corso2023particle} calculates the average pairwise cosine similarity of features within an image batch, with DINO~\cite{caron2021dino} as the feature extractor. 

Table~\ref{tab:sd-diversity} reports the in-batch similarity score to measure the diversity of text-to-image tasks. \textit{We did not use Recall for text-to-image evaluation because it requires generating numerous samples per prompt. Furthermore, we found the Diversity score in \cite{yin2024improved} to be unreliable; higher CFG values in the teacher model, which are known to reduce diversity, will result in better Diversity scores.} Our main finding is that JS outperforms reverse-KL in in-batch similarity across datasets and CFGs (two numbers are MS-COCO/Parti-Prompt subset~\cite{yin2024improved} evaluations). JS shows a larger diversity gain on higher CFG, suggesting it preserves more modes.

\subsection{Image quality evaluation based on HPSv2}
\label{app:hpsv2}
\begin{table}[t]
\footnotesize
    \centering
    \begin{tabular}{l c c c c}
    \toprule
    & Anime & Photo & Concept Art & Painting\\
    \midrule
    SDv1.5 (CFG=3) & 26.30& 27.56 & 25.86 &  26.08\\
    SDv1.5 (CFG=8) & 27.53& 28.46 & 26.94 &  26.83\\
            & \\[-1.9ex]
    \cdashline{1-5}
    & \\[-1.9ex]
    InstaFlow~\cite{liu2023instaflow} & 25.98 & 26.32 & 25.79 & 25.93\\
    SwiftBrush~\cite{nguyen2024swiftbrush} & 26.91 & 27.21 & 26.32 & 26.37\\
    SwiftBrush v2~\cite{dao2025swiftbrush} & 27.25 & 27.62 & 26.86 & 26.77\\
        & \\[-1.9ex]
    \cdashline{1-5}
    & \\[-1.9ex]
        \textit{CFG=1.75} \\
    reverse-KL (DMD2~\cite{yin2024improved}) & 26.20  & 27.33 & 25.82 & 25.68  \\
    JS~(\textit{ours}) &  26.32\textcolor{green}{$\uparrow$} & 27.71\textcolor{green}{$\uparrow$} & 25.79\textcolor{red}{$\downarrow$} & 25.81\textcolor{green}{$\uparrow$} \\
    & \\[-1.9ex]
    \cdashline{1-5}
    & \\[-1.9ex]
        \textit{CFG=5} \\
    reverse-KL (DMD2~\cite{yin2024improved}) & 26.52  &27.86 &26.25  & 26.10  \\
    JS~(\textit{ours}) & 26.85\textcolor{green}{$\uparrow$}  & 27.98\textcolor{green}{$\uparrow$} & 26.37\textcolor{green}{$\uparrow$} & 26.34\textcolor{green}{$\uparrow$} \\
    \bottomrule
    \end{tabular}
    \caption{HPSv2 score}
    \label{tab:hpsv2}
    \vspace{-10pt}
\end{table}

We evaluate the HPSv2 score~(higher is better) following the protocol in \cite{dao2025swiftbrush} to assess the image quality. In Table~\ref{tab:hpsv2}, we observe that JS consistently outperforms reverse-KL on almost all prompt categories. When CFG=5, JS performs competitively to 50-step teacher SDv.15 (CFG=8) and SwiftBrush v2. 

\subsection{Training stability}

In section~\ref{sec:properties} and section~\ref{app:property}, we show that more mode-coverage divergence tends to have a more rapidly increasing $h$, resulting in a higher-variance gradient. Although we propose a double normalization scheme in section~\ref{sec:properties}, we show that it is insufficient on larger-scale COYO-700M when using the SD v.1.5 model. As shown in Fig.~\ref{fig:training-loss}, the loss of forward-KL has a significantly larger fluctuation than the one in JS. In addition, the forward-KL achieves a much worse FID~(8.70) compared to JS~(7.45). We hypothesize that the inaccurate estimation of the density ratio $r$ by the discriminator on this dataset contributes to this phenomenon. We will leave the stabilization techniques for further work.

\begin{figure*}[htbp]
    \centering
    \includegraphics[width=0.5\linewidth]{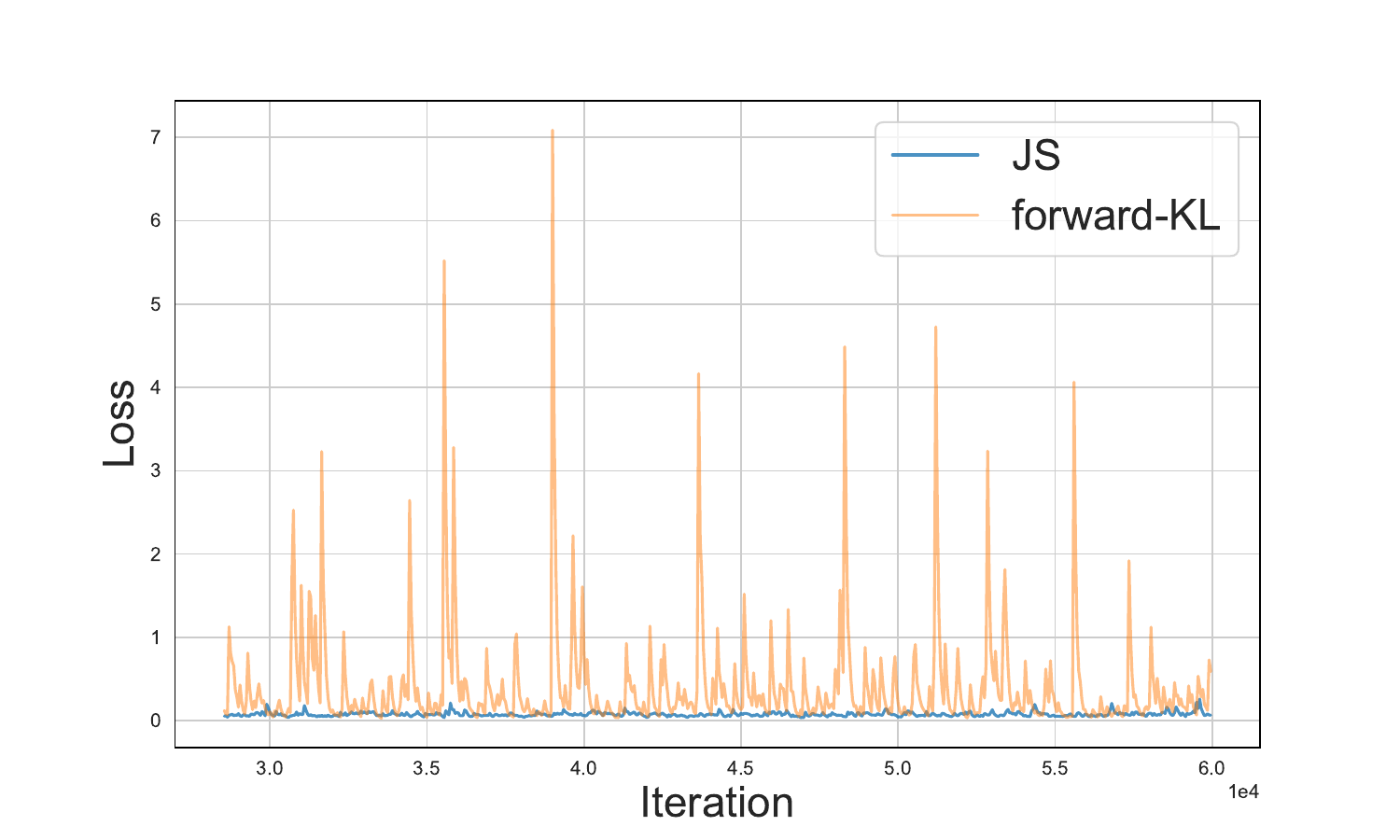}
    \caption{Training dynamics of JS and forward-KL on COYO-700M.}
    \label{fig:training-loss}
\end{figure*}

\subsection{Higher classifier-free guidance}

In this section, we experiment with applying higher classifier-free guidance~(CFG) to \methodtext, by replacing the teacher score in the gradient~(\Eqref{eq:time-0-loss}) with the corresponding CFG version. In Fig.~\ref{fig:vis-cfg-5}, Fig.~\ref{fig:vis-cfg-5-2} and Fig.~\ref{fig:vis-cfg-5-3}, we compare the generated samples by one-step JS with CFG=5, and by SD v1.5~(NFE=50) with CFG=5 / CFG=8. The first three prompts in these figures are randomly chosen from COYO-700M, and the next three prompts are from \cite{yin2024one}. We observe that, in general, the one-step student matches, or even outperforms, the teacher in most cases. In addition, JS produces more diverse samples compared to reverse-KL~(RKL).

We observe that both JS and revesre-KL~(DMD2~\cite{yin2024improved}) diverges when using CFG=8. We first hypothesize that this is because the generated samples and the real data~(COYO-700M) have a larger domain gap, exacerbating the training instability in GAN. DMD~\cite{yin2024one, yin2024improved} uses LAION-Aesthetic 5.5+ as the training set, which is considered more saturated and has a smaller domain gap with data generated by high CFG. However, we find that removing GAN loss does not resolve this issue.The training instability with high CFG might be linked to the weak teacher model~(SD1.5), as our preliminary experiments with clipped teacher score prevented divergence.

\subsection{Ablation study on GAN loss}

\begin{table}[t]
\footnotesize
    \centering
    \begin{tabular}{l c c}
    \toprule
    &  w/o GAN loss & w/ GAN loss \\
    \midrule
   \textit{CIFAR-10}\\
    reverse-KL (DMD2)  & 4.07 & 2.13\\
    JS  & 3.98 & 2.00 \\
    squared Hellinger &3.81 & 1.99 \\
    forward-KL &\bf{3.76} & \bf{1.92}\\
    \midrule
    \textit{MS-COCO-30k}\\
    reverse-KL (DMD2)  & 9.54 & 8.17\\
    JS  & \bf{9.10} & \bf{7.42} \\
    \bottomrule
    \end{tabular}
    \caption{FID score for ablation study on GAN loss}
    \label{tab:gan-loss-ab}
\end{table}

The final training framework in the experimental section contains two losses: the $f$-distill objective~(\Eqref{eq:obj-final}) and GAN loss. In this section, we conduct ablation studies on the GAN loss on CIFAR-10 and MS-COCO 2014 in $f$-distill.As shown in Table~\ref{tab:gan-loss-ab}, the relative ranking of FID scores by different f-divergence remains the same with or without GAN loss across datasets.

\begin{figure*}
\centering
    \begin{subfigure}[b]{0.87\textwidth}
    \includegraphics[width=\textwidth]{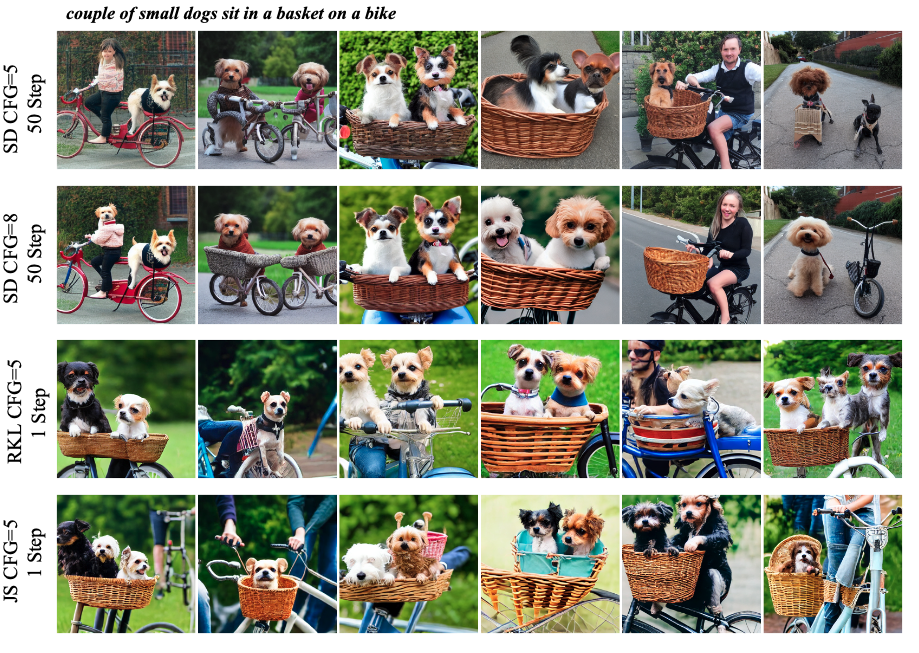}
  \end{subfigure}
  \vfill
  \begin{subfigure}[b]{0.87\textwidth}
    \includegraphics[width=\textwidth]{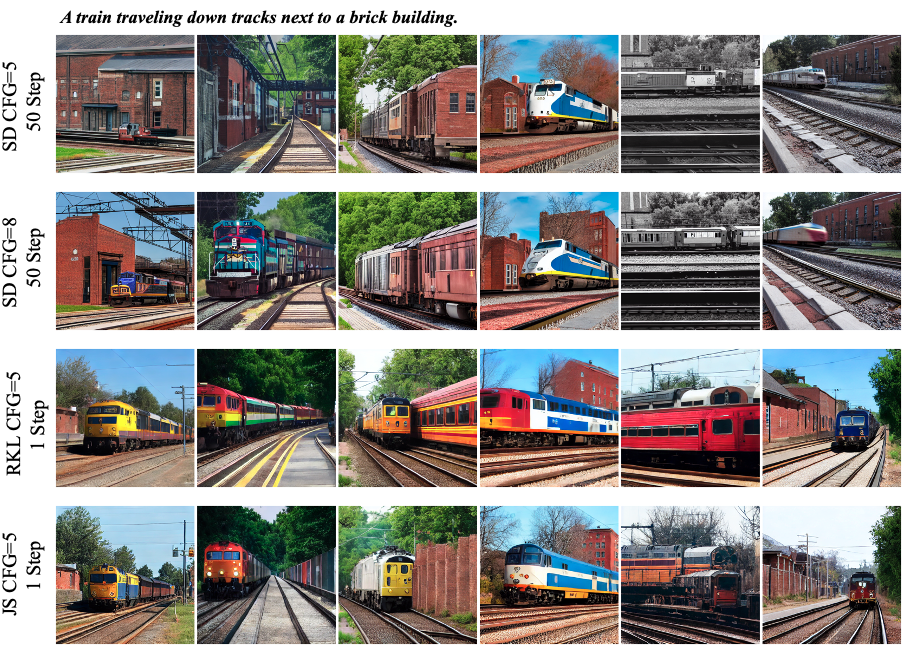}
  \end{subfigure}
  \vspace{-12pt}
    \caption{SD v1.5: Generated samples from multi-step teachers and single-step students, using the same prompts and random seeds. The real data used for GAN objective are from COYO-700M.}
    \label{fig:vis-cfg-5}
\end{figure*}

\begin{figure*}
\centering
    \begin{subfigure}[b]{0.88\textwidth}
\includegraphics[width=\textwidth]{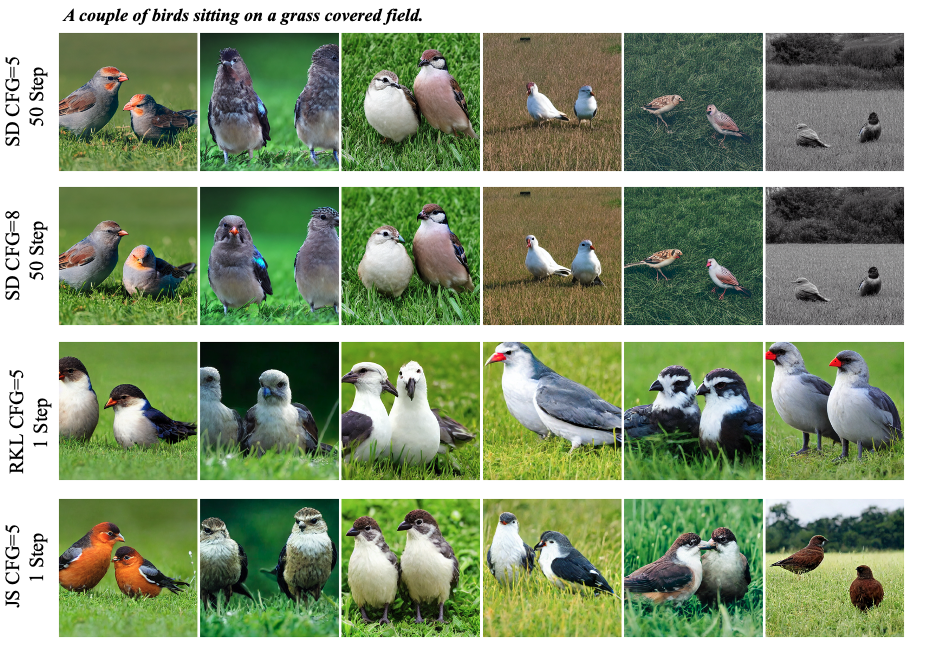}
  \end{subfigure}
\vfill
  \begin{subfigure}[b]{0.87\textwidth}
  \includegraphics[width=\textwidth]{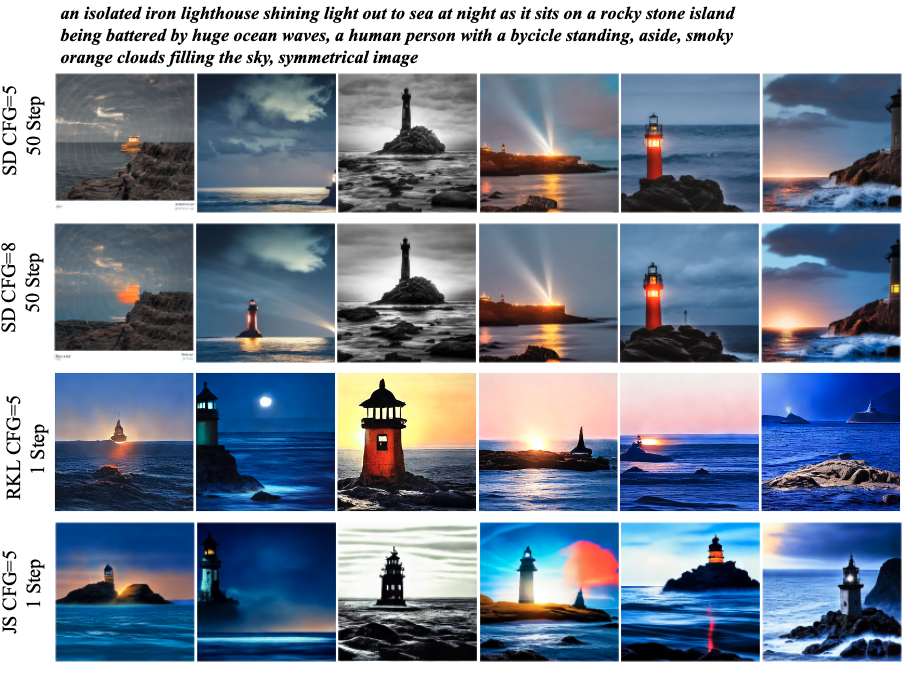}
  \end{subfigure}
  \vspace{-12pt}
    \caption{SD v1.5: Generated samples from multi-step teachers and single-step students, using the same prompts and random seeds. The real data used for GAN objective are from COYO-700M.}
    \label{fig:vis-cfg-5-2}
\end{figure*}

\begin{figure*}
\centering
  \begin{subfigure}[b]{0.875\textwidth}
    \includegraphics[width=\textwidth]{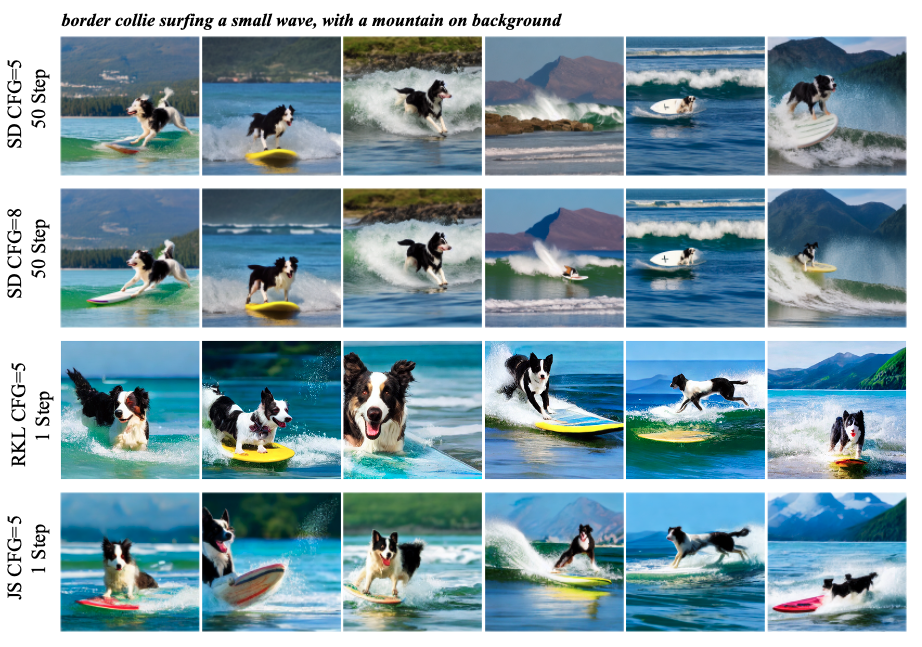}
  \end{subfigure}\vfill
    \begin{subfigure}[b]{0.87\textwidth}
    \includegraphics[width=\textwidth]{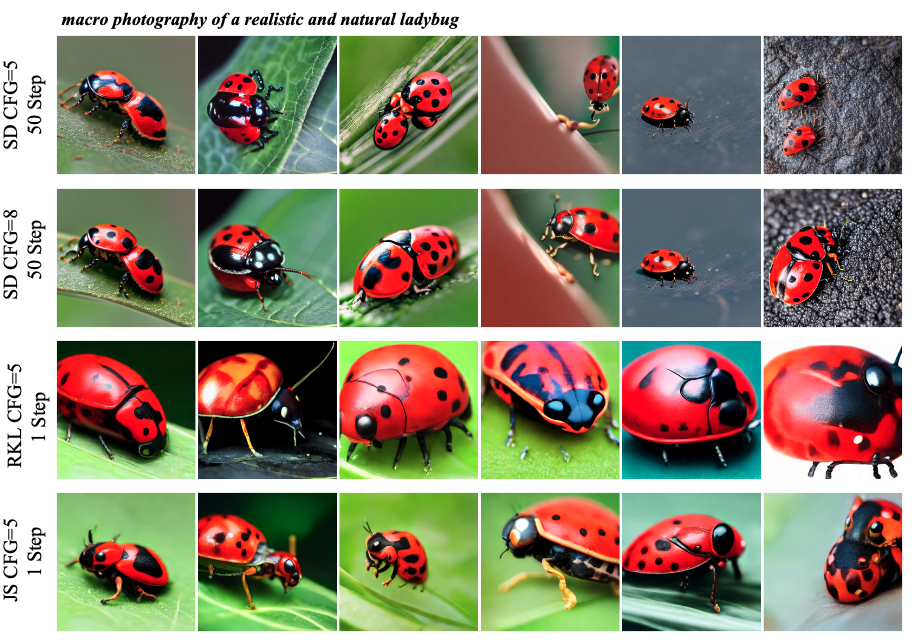}
  \end{subfigure}
  \vspace{-12pt}
    \caption{SD v1.5: Generated samples from multi-step teachers and single-step students, using the same prompts and random seeds. The real data used for GAN objective are from COYO-700M.}
    \label{fig:vis-cfg-5-3}
\end{figure*}

\vspace{-2pt}
\section{Extended Samples}
\vspace{-2pt}
\label{app:extended}

We provide extended samples on CIFAR-10~(Fig.~\ref{fig:cifar-edm}~(multi-step teacher), Fig.~\ref{fig:cifar-js}~(KL, \method)); ImageNet-64~(Fig.~\ref{fig:imagenet-64-edm}~(multi-step teacher), Fig.~\ref{fig:imagenet-64-js}~(JS, \method)); COYO-700M~( Fig.~\ref{fig:sdxl-js-cfg-5}~(SDXL \method), Fig.~\ref{fig:sd-js-cfg-5}~(SDv1.5 \method)). The teacher and student models use the same random seeds and class labels/text prompts. We further provide the randomly sampled COYO-700M prompts for the generated samples in the main text and supplementary material below.
\vspace{-5pt}
\paragraph{8 prompts for generated images in Fig.~\ref{fig:vis-main}:}
\begin{itemize}
    \item There is a long river in the middle of a spectacular valley, and the drone aerial photography shows the mountains and rivers.
    \item A soft beam of light shines down on an armored granite wombat warrior statue holding a broad sword. The statue stands an ornate pedestal in the cella of a temple. wide-angle lens. anime oil painting.
    \item A train ride in the monsoon rain in Kerala. With a Koala bear wearing a hat looking out of the window. There is a lot of coconut trees out of the window.
    \item A Labrador wearing glasses and casual clothes is lying on the bed reading.
    \item A cute little penguin walks on an Antarctic glacier, searching for food.
    \item A teddy bear on a skateboard in times square.
    \item The car is accelerating, the background on both sides is blurred, focus on the body.
    \item A couple of birds sitting on a grass covered field.
\end{itemize}
\vspace{-5pt}
\paragraph{8 prompts for generated images in Fig.~\ref{fig:vis}:}
\begin{itemize}
\item A man putting a pan inside of an oven with his bare hand.
\item A man flying a kite on the beach. 
\item A hotel room filled with beige and blue furniture.
\item A large stone bench sitting next to rose bushes.
\item Clock tower over a crowd of people standing on a bridge.
\item a blue and white passenger train coming to a stop
\item The clock shown above has someone's name on it.
\item A large predatory bird sits on a tree branch in an exhibit.
\end{itemize}

\vspace{-5pt}
\paragraph{24 prompts for generated images in Fig.~\ref{fig:sd-js-cfg-5}:}
\begin{itemize}
\item A large green gate sitting in front of a red brick building.
\item A room in a private house for loosening up and institutionalizing. 
\item A train traveling down tracks next to a brick building.
\item A couple of birds sitting on a grass covered field.
\item A cat is laying on the other side of a cactus.
\item A black horse standing in a desert field surrounded by mountain.
\item Many difference birds in cages on display in an outdoor market area.
\item A group of people watching a man skateboard.
\item A man riding on top of a surfboard in the ocean.
\item a table with some plates of food on it
\item An orange billboard truck driving down a street in front of a crowd of people.
\item A vase filled with lots of different colored flowers.
\item A large living room is seen in this image.
\item A black and white cat that is standing on all fours and has an elephant hat on it's head.
\item Silhouette of a herd of elephants walking across the field
\item Separate men sitting on park benches playing on phone and reading.
\item A wooden table topped with plates and bows filled with food
\item A cat curled up in a sunny spot on a table sleeping.
\item A woman going down the stairs with a backpack on and a suitcase in her hand. 
\item A man riding a red surfboard on a wave in the ocean.
\item a couple of small dogs sit in a basket on a bike
\item A bike parked in front of a parking meter.
\item Two men riding mopeds, one with a woman and boy riding along.
\item A BOY IS ON A SKATE BOARD IN THE COURT
\end{itemize}

\begin{figure*}
    \centering
\includegraphics[width=0.85\linewidth]{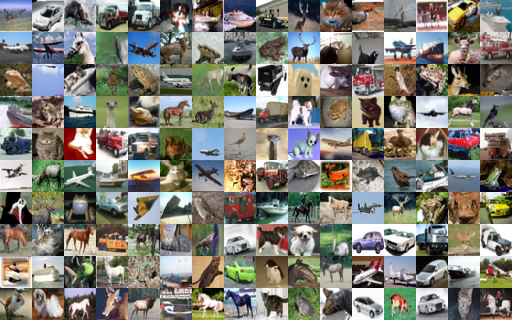}
    \caption{35-step generated CIFAR-10 samples, by EDM~\cite{Karras2022ElucidatingTD}~(\textcolor{red}{teacher}). FID score: 1.79 }
    \label{fig:cifar-edm}
\end{figure*}

\begin{figure*}
    \centering
\includegraphics[width=0.85\linewidth]{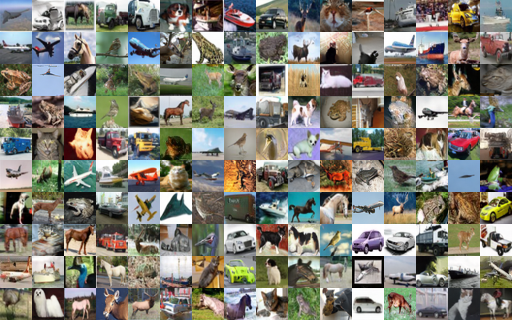}
    \caption{One-step generated CIFAR-10 samples, by KL in \textcolor{blue}{\method}. FID score: 1.92 }
    \label{fig:cifar-js}
\end{figure*}

\begin{figure*}
    \centering
\includegraphics[width=0.8\linewidth]{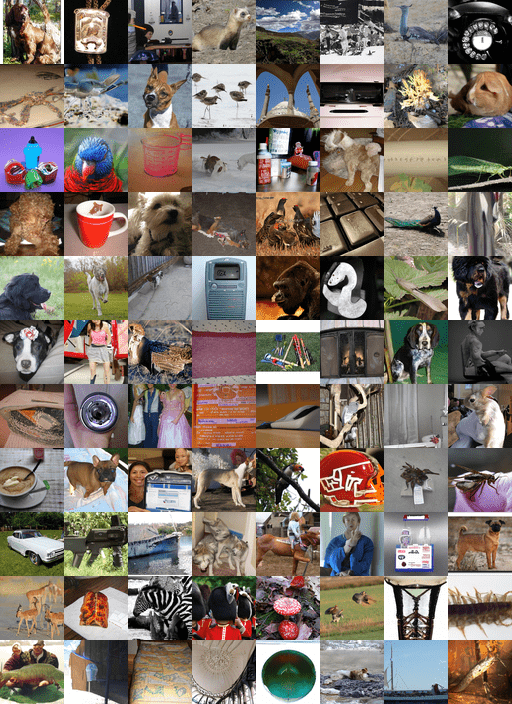}
    \caption{79-step generated ImageNet-64 samples, by EDM~\cite{Karras2022ElucidatingTD}~(\textcolor{red}{teacher}). FID score: 2.35 }
    \label{fig:imagenet-64-edm}
\end{figure*}

\begin{figure*}
    \centering
\includegraphics[width=0.8\linewidth]{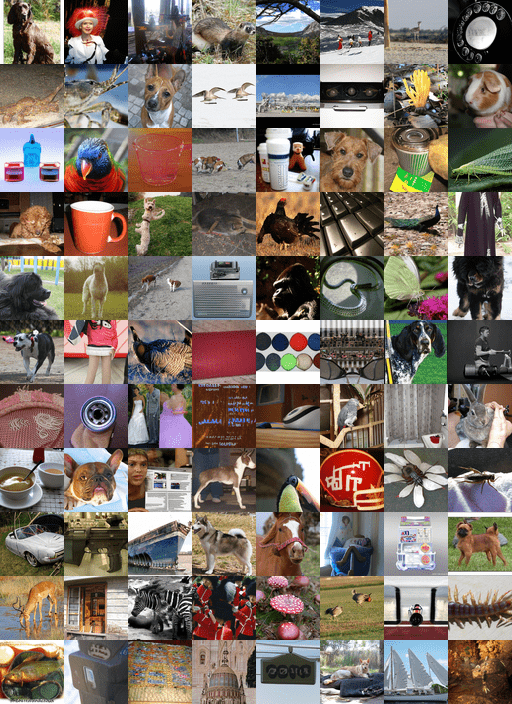}
    \caption{One-step generated ImageNet-64 samples, by JS in \textcolor{blue}{\method}. FID score: 1.16 }
    \label{fig:imagenet-64-js}
\end{figure*}



\begin{figure*}
    \centering
\includegraphics[width=0.8\linewidth]{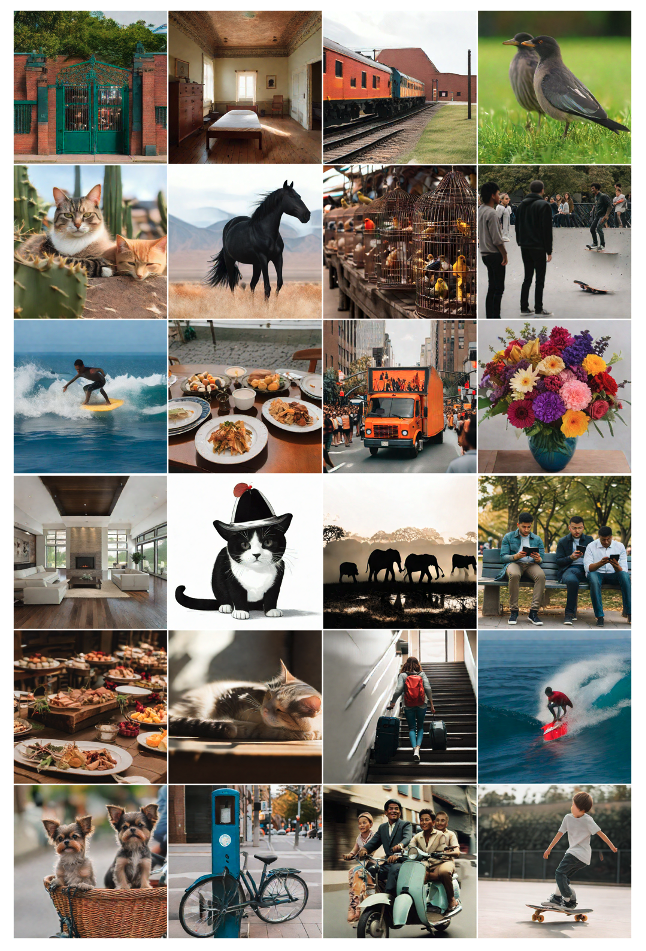}
    \caption{One-step generated SDXL samples, using randomly sampled COYO-700M prompts, by JS in \textcolor{blue}{\method}, with CFG=8.}
    \label{fig:sdxl-js-cfg-5}
\end{figure*}

\begin{figure*}
    \centering
\includegraphics[width=0.8\linewidth]{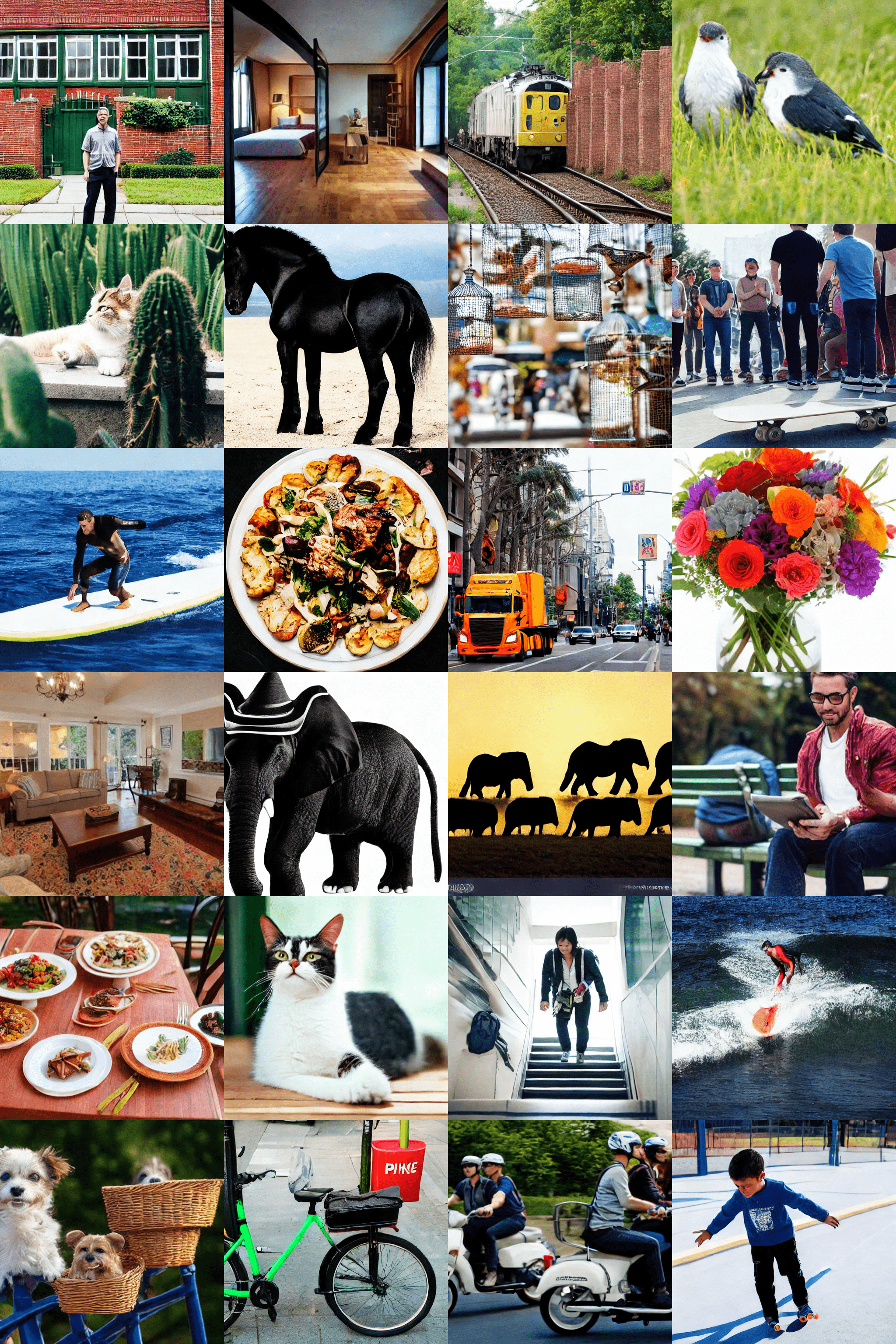}
    \caption{One-step generated SD v1.5 samples, using randomly sampled COYO-700M prompts, by JS in \textcolor{blue}{\method}, with CFG=5.}
    \label{fig:sd-js-cfg-5}
\end{figure*}

\end{document}